\documentclass{article}

\usepackage{microtype}
\usepackage{graphicx}
\usepackage{subfigure}
\usepackage{booktabs} 

\usepackage{hyperref}

\usepackage[accepted]{icml2025}

\usepackage{amsmath}
\usepackage{amssymb}
\usepackage{mathtools}
\usepackage{amsthm}

\usepackage[capitalize,noabbrev]{cleveref}

\theoremstyle{plain}
\newtheorem{theorem}{Theorem}[section]
\newtheorem{proposition}[theorem]{Proposition}
\newtheorem{lemma}[theorem]{Lemma}

\theoremstyle{definition}

\theoremstyle{remark}

\definecolor{electricindigo}{rgb}{0.44, 0.0, 1.0}
\usepackage{multirow}

\definecolor{deblue}{RGB}{11,132,147}
\definecolor{ocra}{RGB}{204, 119, 34}
\usepackage{enumitem}
\usepackage{tikz}

\usepackage{MnSymbol,bbding,pifont}

\definecolor{deblue}{RGB}{11,132,147}
\definecolor{ocra}{RGB}{204, 119, 34}
\definecolor{electricindigo}{rgb}{0.44, 0.0, 1.0}
\usepackage[capitalize]{cleveref}
\crefname{section}{Sec.}{Secs.}
\Crefname{section}{Section}{Sections}
\Crefname{table}{Table}{Tables}
\crefname{table}{Tab.}{Tabs.}
\definecolor{indigo(web)}{rgb}{0.29, 0.0, 0.51}
\usepackage{multirow}
\usepackage{xcolor}
\definecolor{darkblue}{RGB}{40,40,85}
\definecolor{babyblue}{rgb}{0.54, 0.81, 0.94}
\definecolor{pearDark}{HTML}{2980B9}
\definecolor{pearDarker}{HTML}{1D2DEC}
\usepackage[capitalize]{cleveref}
\crefname{section}{Sec.}{Secs.}
\Crefname{section}{Section}{Sections}
\Crefname{table}{Table}{Tables}
\crefname{table}{Tab.}{Tabs.}
\usepackage{fontawesome5}
\usepackage{multirow}
\usepackage{colortbl}
\usepackage{tabulary}
\usepackage{etoolbox}
\usepackage{tikz}
\usepackage{pifont}
\usepackage[textsize=tiny]{todonotes}
\usepackage{xcolor}
\newcommand{\Sam}[1]{{\color{blue} #1}}

\icmltitlerunning{SpectraKAN: Conditioning Spectral Operators}

\begin{document}

\twocolumn[
\icmltitle{SpectraKAN: Conditioning Spectral Operators}




\begin{icmlauthorlist}
\icmlauthor{Chun-Wun Cheng}{yyy}
\icmlauthor{Carola-Bibiane Sch\"onlieb}{yyy}
\icmlauthor{Angelica I. Aviles-Rivero}{zzz}

\end{icmlauthorlist}

\icmlaffiliation{yyy}{DAMTP, University of Cambridge, Cambridge, UK}
\icmlaffiliation{zzz}{YMSC, Tsinghua Univsity, Beijing, China}

\icmlcorrespondingauthor{Angelica I. Aviles-Rivero}{aviles-rivero@tsinghua.edu.cn}

\icmlkeywords{Machine Learning, ICML}

\vskip 0.3in
]



\printAffiliationsAndNotice{} 

\begin{abstract}
Spectral neural operators, particularly Fourier Neural Operators (FNO), are a powerful framework for learning solution operators of partial differential equations (PDEs) due to their efficient global mixing in the frequency domain. However, existing spectral operators rely on \textit{static Fourier kernels} applied uniformly across inputs, limiting their ability to capture multi-scale, regime-dependent, and anisotropic dynamics governed by the global state of the system. We introduce \textbf{SpectraKAN}, a neural operator that {conditions the spectral operator on the input itself}, turning static spectral convolution into an \textit{input-conditioned integral operator}. This is achieved by extracting a compact global representation from spatio-temporal history and using it to modulate a multi-scale Fourier trunk via single-query cross-attention, enabling the operator to adapt its behaviour while retaining the efficiency of spectral mixing. We provide theoretical justification showing that this modulation converges to a resolution-independent continuous operator under mesh refinement and KAN gives smooth, Lipschitz-controlled global modulation. Across diverse PDE benchmarks, SpectraKAN achieves state-of-the-art performance, reducing RMSE by up to $49\%$ over strong baselines, with particularly large gains on challenging spatio-temporal prediction tasks.

\end{abstract}

\section{Introduction}
Partial differential equations (PDEs) underpin a large fraction of scientific and engineering computation, governing phenomena ranging from diffusion--reaction and wave propagation to geophysical flows and climate dynamics \cite{roubivcek2005nonlinear}. While classical numerical solvers (e.g., finite difference/volume/element methods) are accurate and robust, their computational cost can be prohibitive in regimes requiring repeated solves---such as design optimisation, uncertainty quantification, and inverse problems or when high resolutions and long time horizons are needed.

Neural operators \cite{lu2019deeponet, li2020fourier} have recently emerged as a compelling alternative for building fast surrogate PDE solvers by learning mappings between infinite-dimensional function spaces from data. Once trained, a neural operator can evaluate the solution operator orders-of-magnitude faster than a traditional solver, while retaining key benefits such as discretisation-invariant generalisation and the potential for zero-shot super-resolution. Among these approaches, the Fourier Neural Operator \cite{li2020fourier} has become a widely adopted backbone due to its efficient global mixing via spectral convolution.

Despite strong empirical performance, existing Fourier-based operators face persistent challenges on complex spatio-temporal systems.
First, many PDEs exhibit multi-scale structure and anisotropy, while standard FNO blocks rely on single-resolution spectral representations with largely static filters, limiting their ability to jointly capture coarse–fine dynamics and direction-dependent behavior.
Second, real PDE dynamics are often regime-dependent, with global states and constraints influencing local evolution; however, most Fourier operators apply the same spectral kernel at inference, reducing adaptivity and potentially harming stability in long autoregressive rollouts.
Third, interpretability and controllability remain limited, as nonlinearity is typically introduced through fixed pointwise activations with little structural meaning.

A broad set of extensions has been proposed to alleviate parts of these issues. Multi-resolution and multi-basis operator learners \cite{li2024multi, gupta2021multiwavelet, roy2025physics} improve scale separation, and adaptive/state-dependent kernels \cite{dangadaptfno} introduce conditionality through attention or hypernetworks. Transformer-based neural operators \cite{wu2024transolver, hao2023gnot} offer flexible nonlocal interactions, while Kolmogorov--Arnold Networks (KANs) \cite{liu2024kan} replace fixed activations with learnable spline-based functions that can improve interpretability. Yet, these directions are often pursued in isolation: multi-scale variants may remain globally static, adaptive kernels may be single-scale or expensive, and KAN components are frequently used as local nonlinearities without explicitly controlling the operator-level spectral behaviour or global conditioning. As a result, there remains a gap for a unified operator that is simultaneously (i) multi-scale and anisotropy-aware, (ii) globally conditioned in a state-dependent manner, and (iii) efficient and stable for spatio-temporal PDE forecasting.

We address this gap with \textbf{SpectraKAN}, a \textbf{Spectral KAN Attention Neural Operator (SKAN)} that turns a static Fourier operator into a \emph{dynamic, context-aware} spectral operator. SpectraKAN couples three components:
\emph{(i) Adaptive Multi-Scale FNO (AMFNO)} that fuses coarse and fine spectral branches (with directional gating) to better capture long-range coupling while reducing Fourier truncation bias;
\emph{(ii) Global Conditioning KAN (G-KAN)} that compresses spatio-temporal history (and optional parameters) into a compact, spline-interpretable global token; and
\emph{(iii) Global Modulation Layer (GML)} that injects this token through a single-query cross-attention mechanism. Crucially, SpectraKAN does \emph{not} only refine the Fourier convolution kernel within the FNO block. Instead, it introduces an additional \emph{input-conditioned kernel integral operator} that is conditioned on the entire input via the KAN token. SpectraKAN can be viewed as 
$G(u)(x) =f(u)(x) + A_{\Phi_{\mathrm{KAN}}(u)}[f(u)](x)$ 
 where $f(\cdot)$ is the (multi-scale) Fourier trunk and $\mathcal{A}_{{\Phi{\mathrm{KAN}}(u)}}$ is a nonlocal integral operator whose kernel depends on the KAN-produced global context. This separation changes the effective operator family from a purely stationary spectral convolution to a \emph{globally conditioned, non-stationary} integral-kernel operator, enabling the model to adapt its global dynamics across regimes while keeping the spatial backbone efficient.

Our main contributions are as follows.
\begin{itemize}\setlength\itemsep{0em}\setlength\parskip{0em}\setlength\parsep{0em}
    \item We identify a fundamental limitation of existing Fourier neural operators: the use of \emph{static spectral kernels} that are applied uniformly across inputs, restricting their ability to model regime-dependent and multi-scale PDE dynamics.
    \item We introduce \textbf{SpectraKAN}, which reformulates spectral operator learning by augmenting the Fourier trunk with an \emph{input-conditioned kernel integral operator}, thereby changing the effective operator class from stationary spectral convolution to a globally conditioned, non-stationary operator.
    \item We propose a principled mechanism for global conditioning via a Kolmogorov--Arnold Network that extracts a compact, spline-interpretable token from spatio-temporal history and modulates the spectral operator through single-query attention and showed that KAN gives smooth, Lipschitz- controlled global modulation.
    \item We provide theoretical analysis showing that this attention-based modulation is a consistent quadrature approximation of a continuous integral operator and converges to a \emph{resolution-independent operator} under mesh refinement.
    \item We demonstrate across diverse PDE benchmarks that conditioning the spectral operator yields substantial performance gains over strong Fourier-, transformer-, and KAN-based baselines, particularly on challenging spatio-temporal and multi-scale prediction tasks.
\end{itemize}

\section{Related Work}
\paragraph{Deep learning for PDEs} Neural operators have emerged as a powerful framework for PDE surrogates, learning mappings between function spaces with discretization-invariant generalization. DeepONet \cite{lu2019deeponet} parameterizes nonlinear operators via coupled branch–trunk networks, while FNO \cite{li2020fourier} learns integral kernels in the Fourier domain and achieves strong performance on canonical benchmarks, including zero-shot super-resolution. Building on these foundations, numerous extensions have been proposed to improve expressiveness and efficiency, including (i) multi-resolution or multi-basis operators for enhanced multi-scale modeling \cite{kossaifi2023multi,you2024mscalefno,guo2024mgfno}, (ii) factorized and U-shaped architectures that improve depth and computational efficiency \cite{wen2022u,rahman2022u,li2023long,wang2025hybrid}, and (iii) physics-informed neural operators that incorporate PDE residual constraints to enhance data efficiency and out-of-distribution generalization \cite{li2024physics,ding2025physics}.

\textbf{Transformer-based neural operator.}
Transformer-based neural operators\cite{vaswani2017attention} leverage attention to model long-range dependencies in time-dependent PDEs, with variants such as Galerkin Transformer\cite{cao2021choose} improving efficiency via linear-complexity attention, OFormer\cite{li2022transformer} and GNOT\cite{hao2023gnot} extending the framework to sampled functions, irregular meshes, and multi-input settings while several variants improve accuracy or efficiency \cite{wu2024transolver,bryutkin2024hamlet,wang2024cvit}. More recently, Mamba \cite{cheng2025mamba,hu5149007deepomamba}. has emerged as an efficient alternative and has been applied to operator learning. In contrast, rather than relying on full token-wise self-attention as the core operator, our method uses a multi-scale gated Fourier trunk modulated by a single-query cross-attention mechanism driven by a KAN-based global token, yielding an input-conditioned integral operator that globally controls spatial dynamics.

\textbf{KAN.} Kolmogorov–Arnold Networks (KAN) use B-spline–based learnable activations to improve accuracy and interpretability. KANO \cite{lee2025kano} introduces a KAN-based neural operator motivated by interpretability, while physics-informed KAN integrates KAN into PINNs to enhance accuracy and generalization \cite{zhang2025physics}. Other works incorporate KAN into PDE solvers to improve efficiency \cite{yeo2024kan, abueidda2025deepokan}.
Existing methods primarily embed KAN within neural operators or replace MLP components; in contrast, we are the first to use KAN to learn a global conditioning token.

\section{Methodology}
\textbf{Problem Statement.}
Many PDEs have solutions that are time-dependent, spatially distributed fields. 
We denote the PDE solution as a vector-valued function
$
v : T \times S \times \Theta \rightarrow \mathbb{R}^{d},
$
where $S \subset \mathbb{R}^{p}$ is the spatial domain, $T$ is the temporal index set, and $\Theta$ denotes a parameter space encoding physical coefficients or boundary/initial conditions. 
For each time $t$, we view $v_\theta(t,\cdot)$ as an element of a function space $X$ over $S$ (e.g., $X=L^2(S;\mathbb{R}^d)$).
For a fixed parameter $\theta \in \Theta$, the evolution of the PDE solution is governed by a forward propagator
$
F_{\theta} : X \rightarrow X,
\quad
v_{\theta}(t+1,\cdot) = F_{\theta}\big(v_{\theta}(t,\cdot)\big),
$
which maps the current state to the next state. 
However, because PDE dynamics depend on temporal derivatives, a single state snapshot is often insufficient to approximate $F_{\theta}$. 
Instead, we use a window of $\ell$ preceding states. The resulting discrete-time operator is
$\mathring{F}_{\theta} :
\big( v_{\theta}(t-\ell+1,\cdot), \ldots, v_{\theta}(t,\cdot) \big)
\;\longmapsto\;
v_{\theta}(t+1,\cdot)$.
For convenience, we write the input window as $v_{\theta}([t-\ell+1:t],\cdot)$.
Our goal is to learn a neural operator $\widehat{F}$ such that
$
\widehat{F} \approx \mathring{F}_{\theta} 
\quad \text{for all } \theta \in \Theta,
$
enabling accurate prediction of future PDE states from past observations and generalisation across different physical configurations.

\begin{figure*}[t]
    \centering
    \includegraphics[width=\textwidth]{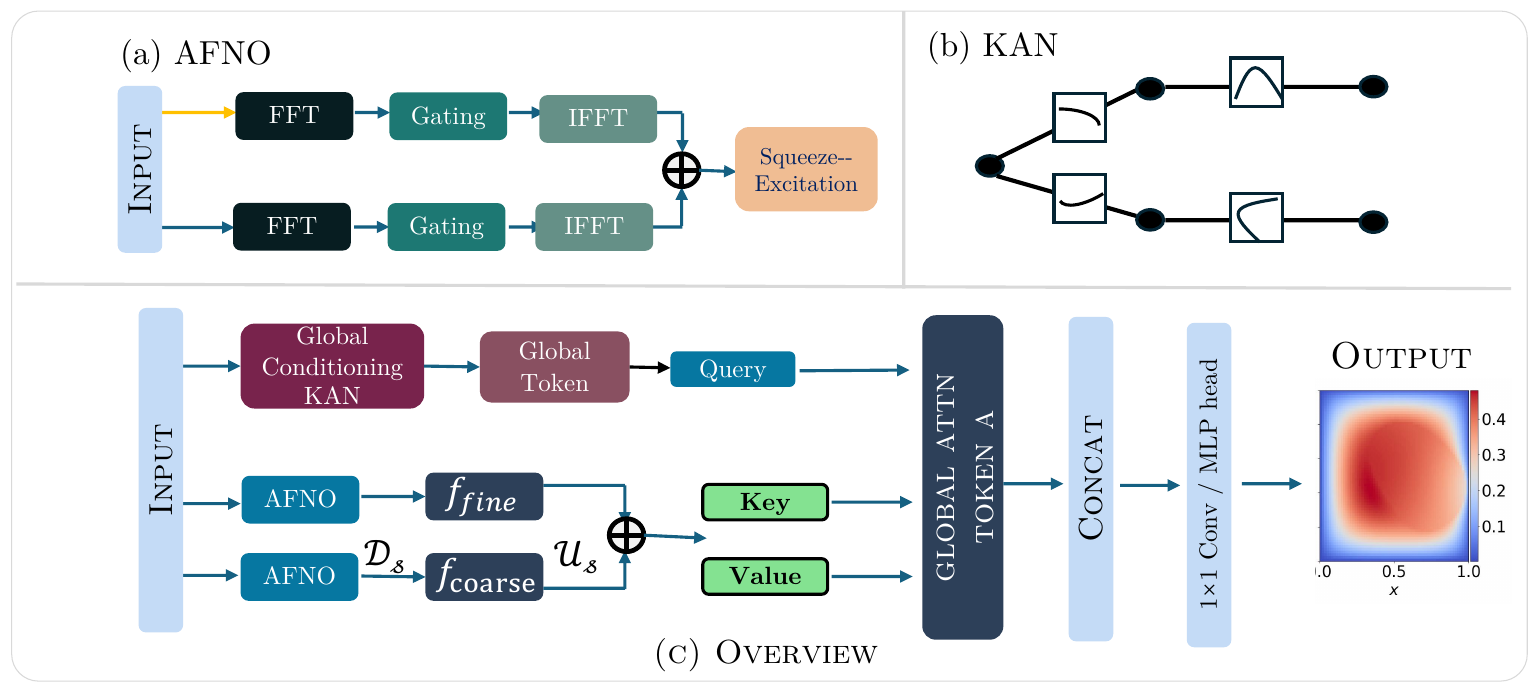}
    \caption{Overview of our SpectraKAN architecture. (a) AFNO performs spectral mixing via FFT–gating–IFFT with Squeeze–Excitation. (b) KAN models nonlinear interactions with learnable spline functions. (c) A global-conditioned KAN produces a global token (query) that fuses with multi-scale AFNO features (key/value), followed by a lightweight projection to obtain the prediction.}
    \label{teaser}
\end{figure*}

\subsection{Architecture Overview}
Our architecture, SpectralKAN (SKAN), consists of three key components.
Adaptive Multi-Scale FNO (AMFNO) captures local to global dynamics by jointly modeling fine and coarse spectral modes, improving long-range coupling while reducing Fourier truncation bias.
Global Conditioning KAN (G-KAN) extracts a low-dimensional global representation of the system’s temporal and parametric state, acting as a latent controller for the spatial operator.
Global Modulation Layer (GML) fuses the AMFNO spatial features with the G-KAN global context, introducing a nonlocal, input-conditioned coupling that turns the spectral trunk into a context-aware integral operator.
Figure~\ref{teaser} illustrates the overall architecture.
Mathematically, we define the neural operator as: \begin{equation}
    G(u)(x) =
\underbrace{f(u)}_{\text{multi-scale spectral operator}} + \underbrace{A_{\Phi_{\mathrm{KAN}}(u)}[f(u)](x)}_{\text{input-conditioned attention operator}} 
\end{equation}

where the global KAN token is $\Phi_{\mathrm{KAN}}(u) = t$, and the input-conditioned attention operator is defined by $A_t[u](x)
=
\sum_{y \in D}
\alpha(x,y;t)\, u(y)$ ,
with attention weights
$\alpha(x,y;t)
=
\mathrm{softmax}\!\left(
\frac{ Q(t) \cdot K(f(y)) }{d}
\right)$.

\subsection{Adaptive Multi-Scale FNO (AMFNO) }
\label{sec:multi-scale-fno}
\paragraph{Spectral convolution with gated axial spectra.}
Each Fourier layer applies separable spectral mixing along both spatial axes.
Let $v\in\mathbb{R}^{X\times Y\times d}$ be a feature field in spatial coordinate domain where X,Y: number of spatial grid points along x and y and d is the number of feature channels. Our goal is to perform a spectral convolution that separately along each axis, with gating weights to modulate frequency responses dynamically.

We first apply the lifting and input embedding. Then we perform a separable 1D Fourier transform. We first consider the $y$-axis. For each fixed $x_i$, we take a one-dimensional discrete Fourier transform (DFT) applied channel-wise along the $Y$ samples of $v$: $\hat{v}_y(x_i, k_2, c) = \sum_{j=0}^{Y-1} v(x_i, y_j, c) \, e^{-2\pi i j k_2 / Y}, k_2 = 0, 1, \ldots, Y-1$,
for each channel $c=1,\ldots,d$. Then we obtain a complex-valued tensor $\hat{v}_y \in \mathbb{C}^{X \times Y \times d}$. However, we retain only the lowest $m$ Fourier modes, i.e.
$k_2 = 0,\ldots,m-1$.
Thus, the truncated representation becomes $\hat{v}_y \in \mathbb{C}^{X \times m \times d}$.
The spectral convolution acts per frequency mode as a complex matrix multiplication across feature channels.  
Let $\hat{v}_y(x_i, k_2) \in \mathbb{C}^d$ denote the feature vector at mode $k_2$.  
Then, the spectral mixing is defined as $\tilde{v}_y(x_i, k_2) = W^{(y)}(k_2) \, \hat{v}_y(x_i, k_2),
\quad W^{(y)}(k_2) \in \mathbb{C}^{d \times d}$.
Each $W^{(y)}(k_2)$ is a complex-valued linear operator that learns how the $d$ channels interact at that frequency.

We apply gating mechanism to make the spectral response adaptive and frequency-aware, we multiply each mode by a learnable real-valued scalar gate 
$\gamma_y(k_2) \in \mathbb{R}$.  
Hence, the gated spectral convolution along the $y$-direction becomes:
$\hat{v}_y'(x_i, k_2) 
= \gamma_y(k_2) \, W^{(y)}(k_2) \, \hat{v}_y(x_i, k_2),
\quad k_2 = 0, 1, \ldots, Y-1$.
Equivalently, in component-wise form across channels:
$\hat{v}_{y,\ell}'(x_i, k_2)
= \gamma_y(k_2) 
\sum_{j=1}^{d} 
W^{(y)}_{\ell j}(k_2) \, 
\hat{v}_{y,j}(x_i, k_2), \ell = 1, \ldots, d$.
Then we get the following equations:
$\hat{v}_y(k_2)
\;\mapsto\;
\sum_{j=1}^{d}
\hat{v}_{y,j}(k_2) \,
W^{(y)}_{j \to \cdot}(k_2) \,
\gamma_y(k_2),
\quad k_2 = 0, 1, \ldots, Y-1$.
Here $W^{(y)}_{j \to \cdot}(k_2)$ denotes the vector of weights mapping the
$j$-th input channel to all output channels at Fourier mode $k_2$. Finally, inverser fourier tranfomer was applied. We then repeat the entire process for x directions.

Then the Squeeze--Excitation block performs adaptive channel re-scaling based on global statistics
of the current feature map $v_{\mathrm{spec}}\in\mathbb{R}^{X\times Y\times d}$.
First, a \emph{squeeze} operation computes a global descriptor
for each channel via average pooling:
\begin{equation}
s_c = \frac{1}{XY}\sum_{i=1}^{X}\sum_{j=1}^{Y}
v_{\mathrm{spec}}(x_i, y_j, c),
\qquad c=1,\dots,d.
\end{equation}
These aggregated descriptors $s = [s_1,\ldots,s_d]\in\mathbb{R}^d$
summarize the global activation of each channel.
Then, an \emph{excitation} operation learns per-channel modulation weights
through a small gating network:
\begin{equation}
a = \sigma\!\big(W_2\,\delta(W_1\,s)\big),
\qquad a\in(0,1)^d,
\end{equation}
where $W_1\in\mathbb{R}^{d\times \frac{d}{r}}$ and
$W_2\in\mathbb{R}^{\frac{d}{r}\times d}$ are the learnable parameters, $\delta(\cdot)$ is the ReLU activation,
and $\sigma(\cdot)$ is the sigmoid function.
Finally, the feature map is reweighted channel-wise:
\begin{equation}
\mathrm{SE}(v_{\mathrm{spec}})(x_i,y_j,c)
= a_c \cdot v_{\mathrm{spec}}(x_i,y_j,c),
\end{equation}
thereby adaptively emphasizing or suppressing channels
depending on their global relevance.

\paragraph{Fine-scale operator.}
To capture high-resolution local dynamics, we stack $L$ gated spectral convolution layers,
each consisting of an axial Fourier operator followed by a pointwise nonlinearity.
This defines the fine-scale Fourier neural operator
$\mathcal{F}_{\mathrm{fine}}
\;=\;
\mathcal{K}_L \circ \sigma \circ \cdots \circ \mathcal{K}_1 ,
\label{eq:fine_operator}$
where each $\mathcal{K}_\ell$ denotes a gated spectral integral operator acting on
fields in $L^2(D;\mathbb{R}^d)$.
Given an input field $u$ and spatial grid $g$, the fine-scale representation is
\[
f_{\mathrm{fine}} = \mathcal{F}_{\mathrm{fine}}(u, g)
\in \mathbb{R}^{X \times Y \times C_{\mathrm{fno}}}.
\]

\paragraph{Coarse branch and fusion.}
To enlarge the receptive field and capture long-range interactions,
we introduce a second Fourier operator acting on a coarse spatial grid.
Let $\mathcal{D}_s$ denote spatial downsampling by factor $s$ and
$\mathcal{U}_s$ the corresponding upsampling operator.
The coarse-scale operator is defined as
\begin{equation}
f_{\mathrm{coarse}}
=
\mathcal{U}_s\!\left(
\mathcal{F}_{\mathrm{coarse}}\!\left(
\mathcal{D}_s u, \mathcal{D}_s g
\right)
\right)
\in \mathbb{R}^{X \times Y \times C_{\mathrm{fno}}}.
\label{eq:coarse_operator}
\end{equation}
The final multi-scale representation is obtained via additive fusion:
\begin{equation}
f(u)
=
f_{\mathrm{fine}} + f_{\mathrm{coarse}}.
\label{eq:multiscale_fusion}
\end{equation}

AMFNO replaces joint spectral mixing with axis-wise separable Fourier operators, reducing memory usage and improving computational efficiency.
This design is particularly effective for \emph{anisotropic PDEs} (e.g., advection--diffusion, Navier--Stokes, heat transport), where dynamics differ across spatial directions.
In the Fourier domain, such anisotropy manifests as direction-dependent energy distributions over $(k_x, k_y)$. By learning separate spectral operators per axis, AMFNO enables direction-specific frequency modeling, which is crucial for capturing anisotropic turbulence, shear flows, and structured diffusion.

\subsection{ Global Conditioning KAN}

\paragraph{KAN-based Global Token Conditioning.}
To enable global, input-conditioned modulation of the spatial operator, we introduce
a KAN-based global token that encodes temporal, parametric, and system-level context
into a compact latent representation.
This token serves as a conditioning variable that governs how global information
is injected into the spatial dynamics via attention.

\paragraph{Global context extraction.}
Let $u_{0:\tau_{\mathrm{in}}}\in\mathbb{R}^{X\times Y\times \tau_{\mathrm{in}}C}$
denote the recent spatio-temporal history of the PDE solution.
We first extract a compact global summary $z = \mathcal{P}\!\left(u_{0:\tau_{\mathrm{in}}}\right)
\in \mathbb{R}^{d_0},
\qquad d_0=\tau_{\mathrm{in}}C$,
where $\mathcal{P}$ denotes a pooling or sampling operator
(e.g., center-pixel extraction or spatial averaging).
This vector captures coarse temporal trends and global state information
that influence the evolution of the solution across the entire domain.

\paragraph{Kolmogorov--Arnold Network (KAN) encoder.}
We map the global summary $z$ into a low-dimensional token
$t\in\mathbb{R}^{C_{\mathrm{kan}}}$ using a Kolmogorov--Arnold Network (KAN),
which represents multivariate functions as compositions of sums of univariate functions.
A KAN of depth $L$ and shape $[d_0,d_1,\dots,d_L]$ defines the mapping
\[
t
=
\Phi_{\mathrm{KAN}}(z)
=
\left(\Phi_{L-1}\circ\cdots\circ\Phi_0\right)(z),
\qquad t\in\mathbb{R}^{d_L}.
\]

Each KAN layer $\Phi_\ell:\mathbb{R}^{d_\ell}\to\mathbb{R}^{d_{\ell+1}}$
is defined component-wise as
\begin{equation}
x^{(\ell+1)}_j
=
\sum_{i=1}^{d_\ell}
\phi^{(\ell)}_{j,i}\!\left(x^{(\ell)}_i\right),
\qquad
j=1,\dots,d_{\ell+1},
\label{eq:kan-layer}
\end{equation}
where $\phi^{(\ell)}_{j,i}:\mathbb{R}\to\mathbb{R}$
are learnable univariate edge functions.

\paragraph{Spline-based edge parametrization.}
Following the standard KAN formulation, each edge function
$\phi^{(\ell)}_{j,i}$ is parameterized as a residual B-spline expansion:
\begin{equation}
\phi^{(\ell)}_{j,i}(x)
=
w^{(\ell)}_{b,j,i}\, b(x)
+
w^{(\ell)}_{s,j,i}
\sum_{r=1}^{G+k}
c^{(\ell)}_{j,i,r}\,
B^{(k)}_{r}(x;\,\Xi^{(\ell)}_i),
\label{eq:kan-spline}
\end{equation}
where: $b(x)=\mathrm{SiLU}(x)$ is a fixed smooth base function, $B^{(k)}_{r}(\cdot;\Xi^{(\ell)}_i)$ are order-$k$ B-spline basis functions defined on a learnable knot grid $\Xi^{(\ell)}_i$, $c^{(\ell)}_{j,i,r}$ are spline coefficients, $w^{(\ell)}_{b,j,i}, w^{(\ell)}_{s,j,i}$ are learnable scaling parameters.

This construction yields a smooth, piecewise-polynomial representation
with explicit control over regularity and approximation capacity.

\paragraph{Approximation and interpretability.}
By classical spline approximation theory, the B-spline expansion in
\eqref{eq:kan-spline} can approximate any univariate function
$f\in C^{k+1}([a,b])$ with error
\[
\left\| f - \sum_r c_r B^{(k)}_r(\cdot;\Xi) \right\|_\infty
=
\mathcal{O}(h^{k+1}),
\]
where $h$ denotes the knot spacing.
Consequently, $\Phi_{\mathrm{KAN}}$ provides a smooth, interpretable,
and low-rank embedding of global dynamical information,
with the learned knot locations $\Xi$ offering direct insight
into which regimes of the input are most relevant.

\paragraph{Global token as operator conditioning.}
The token \(t=\Phi_{\mathrm{KAN}}(z)\) does not directly predict the spatial field, but instead serves as a conditioning variable for the spatial operator. Specifically, \(t\) parameterizes the query in a cross-attention module, defining an input-conditioned integral kernel
\[
\kappa_t(x,y)
=
\operatorname{softmax}_y\!\left(
\frac{Q(t)\,K(f(y))^\top}{\sqrt{d_a}}
\right),
\]
which acts on the multi-scale Fourier features. This mechanism injects global temporal and parametric context into all spatial locations, allowing the operator to adapt to the system state. This global KAN conditioning complements the translation-invariant spectral kernels of the multi-scale FNO: while FNO captures stationary spatial structure, the KAN token modulates the operator in a data-dependent manner, enabling non-stationary dynamics, regime changes, and global constraints beyond spectral convolution alone.

\subsection{Global Modulation Layer}
Global Modulation Layer (GML) as an input-conditioned integral operator.
Let $D\subset\mathbb{R}^2$ denote the spatial domain and let 
$f(u):D\to\mathbb{R}^{C_{\mathrm{fno}}}$ be the multi-scale Fourier feature field produced by AMFNO,
and $t=\Phi_{\mathrm{KAN}}(u)\in\mathbb{R}^{C_{\mathrm{kan}}}$ be the global token produced by G-KAN.
The role of GML is to fuse global context $t$ with local Fourier features $f(u)$ by constructing an
\emph{input-conditioned} nonlocal operator acting on the field. We define the attention-based component as a kernel integral operator
\begin{equation}
\label{eq:gml_integral}
(A_t[u])(x)
=
\int_{D} \kappa_t(x,y)\, \mathcal{V}\!\big(f(u)(y)\big)\,dy,
\qquad x\in D,
\end{equation}
where $\mathcal{V}:\mathbb{R}^{C_{\mathrm{fno}}}\to\mathbb{R}^{d_a}$ is the value projection (implemented by a linear map),
and $\kappa_t(x,y)$ is an input-dependent kernel induced by cross-attention. The attention kernel is then defined as a normalized exponential similarity:
\begin{equation}
\label{eq:kernel_softmax_continuous}
\kappa_t(x,y)
=
\frac{
\exp\!\left(\frac{\langle q(\Phi_{\mathrm{KAN}}(t)),\, k(y)\rangle}{\sqrt{d_a}}\right)
}{
\int_{D}
\exp\!\left(\frac{\langle q(\Phi_{\mathrm{KAN}}(t)),\, k(\xi)\rangle}{\sqrt{d_a}}\right)\, d\xi
}.
\end{equation}

On a grid $\{y_j\}_{j=1}^{N}$ with $N=XY$, the integral in~\eqref{eq:gml_integral}
is approximated by quadrature (Riemann sum):

\begin{equation}
\begin{aligned}
(A_t[u])(x_i)
&\approx
\sum_{j=1}^{N}
\alpha_{ij}(t)\, \mathcal{V}\!\big(f(u)(y_j)\big), \\
\alpha_{ij}(t)
&=
\frac{
\exp\!\left(\frac{\langle q(t),\, k(y_j)\rangle}{\sqrt{d_a}}\right)
}{
\sum_{\ell=1}^{N}
\exp\!\left(\frac{\langle q(t),\, k(y_\ell)\rangle}{\sqrt{d_a}}\right)
}.
\end{aligned}
\end{equation}

which matches the standard softmax($QK^\top$)$V$ attention.

\paragraph{Broadcast fusion and pointwise readout.}
The global attention output in~\eqref{eq:gml_integral} is broadcast back to all spatial locations
and fused with the local Fourier features via a pointwise readout map $\mathcal{R}$ (a $1\times 1$ Conv-MLP): $\hat{u}(x) = \mathcal{R}\!\left(\, f(u)(x)\,,\, A_t[u](x)\,\right),
\qquad x\in D$,
yielding the final context-aware prediction field.

Combining AMFNO and GML, the SKAN neural operator can be summarized as
\begin{equation}
\label{eq:skan_summary}
G(u)(x)
= {f(u)(x)} + A_{\Phi_{\mathrm{KAN}}(u)}[f(u)](x),
\qquad x\in D,
\end{equation}
where $A_{\Phi_{\mathrm{KAN}}(u)}$ is the input-conditioned integral operator defined by
\eqref{eq:gml_integral}--\eqref{eq:kernel_softmax_continuous}.

\begin{table*}[t!]
  \centering
  \caption{ Performance comparison among our method and other representative baselines via different datasets. The best-performing results are highlighted in \colorbox[HTML]{BBFFBB}{green}, while the second-best results are indicated in \Sam{blue} colour.}
  \resizebox{\textwidth}{!}{
    \begin{tabular}{lccccc}
      \toprule
      \textbf{Method} 
      & \textbf{CNS (1D)} 
      & \textbf{Diffusion--Reaction (1D)} 
      & \textbf{Darcy Flow (2D)} 
      & \textbf{Shallow Water (2D)} 
      & \textbf{Climate Modelling (2D)} \\
      \midrule

      U-Net (2015)
      & 12.56934 
      & 0.03812 
      & \Sam{0.01315} 
      & 0.11974 
      & 1.01260 \\

     FNO (2021)
      & 0.34053 
      & 0.02884 
      & 0.02575 
      & 0.01448 
      & 0.01536 \\

      FFNO (2023)
      & \Sam{0.29493} 
      & 0.01409 
      & 0.02482 
      & 0.01191 
      & -- \\

      LNO (2024)
      & 8.97545 
      & 0.04703
      & 0.03538 
      & 0.03938 
      & - \\

      Transolver (2024)
      & 0.72094 
      & \Sam{0.00686} 
      & 0.01987 
      & 0.00586 
      & 0.18575 \\

      LocalFNO (2024)
      & -- 
      & -- 
      & 0.02051 
      & \Sam{0.00438} 
      & \Sam{0.01508} \\

      DeepOKAN (2025)
      & 9.21355 
      & 0.04387 
      & 0.01750
      & 0.08618
      & 0.036257 \\

      CViT (2025)
      &  0.62194
      &  0.00678
      &  0.02287
      &  0.01178
      &  0.024066\\

      \midrule
      \textbf{ SpectraKAN (Ours)} 
      & \cellcolor[HTML]{BBFFBB}\textbf{0.14544} 
      & \cellcolor[HTML]{BBFFBB}\textbf{0.00397} 
      & \cellcolor[HTML]{BBFFBB}\textbf{0.01260} 
      & \cellcolor[HTML]{BBFFBB}\textbf{0.0024} 
      & \cellcolor[HTML]{BBFFBB}\textbf{0.0096119} \\

      \midrule
      \textbf{Improvement (\%)} 
      & 49\% 
      & 42\% 
      & 4.18\% 
      & 45.21\% 
      & 36.26\% \\

      \bottomrule
    \end{tabular}
  }
  \label{tab:main_rmse}
\end{table*}

\begin{table}[t]
\centering
\caption{Relative $L^2$ errors on the Shallow Water Equations across different . Best results are highlighted in \colorbox[HTML]{BBFFBB}{green}}
\label{tab:shallow_water_2d_mean}
\resizebox{\columnwidth}{!}{
\begin{tabular}{lccc}
\toprule
Method & $\rho$ & $u$ & $v$ \\
\midrule
U-Net  & $1.20\times10^{-2}$ & $3.62\times10^{-1}$ & $3.63\times10^{-1}$ \\
FNO    & \Sam{$1.15\times10^{-3}$} & \Sam{$3.20\times10^{-2}$} & $3.33\times10^{-2}$ \\

LNO    & $1.85\times10^{-3}$ & $6.77\times10^{-2}$ & $7.14\times10^{-2}$ \\
Trasnolver & $2.41\times10^{-2}$ & $1.00\times10^{0}$ & $1.00\times10^{0}$ \\
DeepOKAN & $2.35\times10^{-2}$ & $1.01\times10^{0}$ & \Sam{$2.88\times10^{-2}$} \\
CViT & $2.38\times10^{-2}$ & $1.00\times10^{0}$ & $1.00\times10^{0}$ \\
SpectraKAN (Ours)  & \cellcolor[HTML]{BBFFBB}\textbf{$\mathbf{6.85\times10^{-4}}$} & \cellcolor[HTML]{BBFFBB}\textbf{$\mathbf{1.91\times10^{-2}}$} & \cellcolor[HTML]{BBFFBB}\textbf{$\mathbf{2.04\times10^{-2}}$} \\
\bottomrule
\end{tabular}
}
\label{swrl2}
\end{table}

\begin{figure*}[t]
    \centering
    \includegraphics[width=\textwidth]{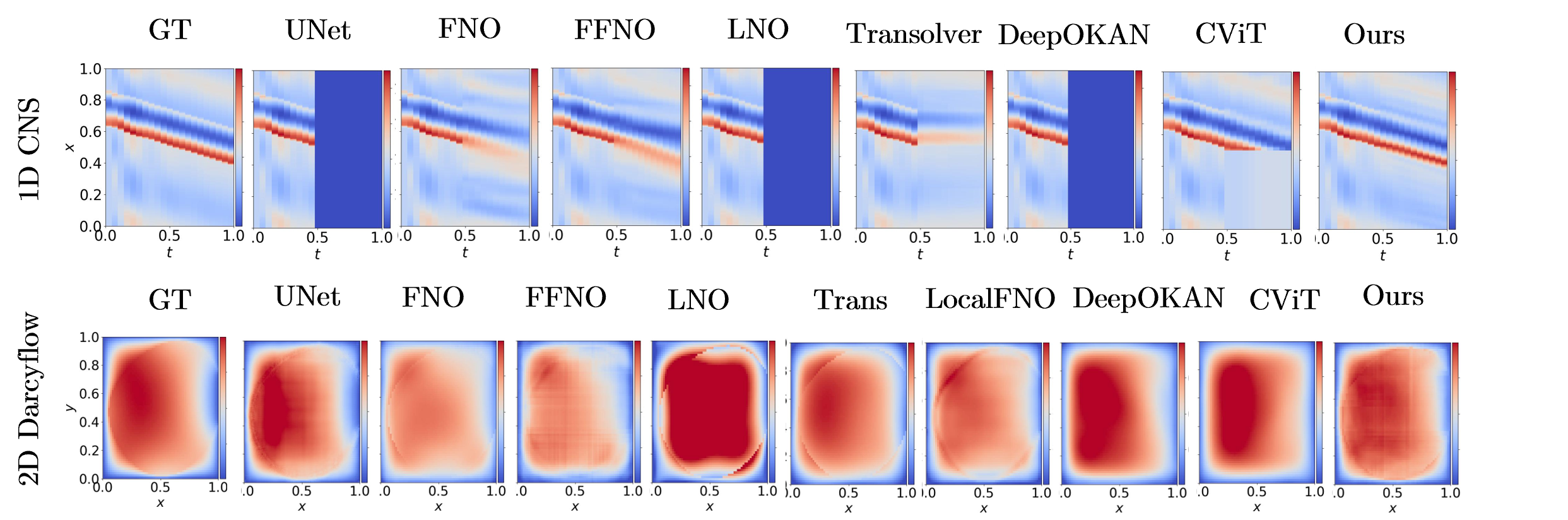}
    \caption{ Top: 1D compressible Navier–Stokes solution evolution (space–time map). Bottom: 2D Darcy flow solution snapshot. While several baselines recover the coarse structure, they often suffer from smoothing or artifacts; our method produces predictions that most closely resemble the ground truth across both settings.}
    \label{fig:fu1}
\end{figure*}

\subsection{Theoretical foundation}

In this section, we provide the theoretical foundation of our method, including its resolution invariance and the choice of the Global KAN token. We also show in the Appendix \ref{theosupp} that SpectraKAN defines a neural operator between function spaces.

\begin{lemma}[Edge Lipschitz bound for spline KAN]
Consider an edge function of the form $\phi(x) = w_b\, b(x) + w_s \sum_{r} c_r\, B_r^{(k)}(x;\Xi)$,
where $b \in C^1$, and $B_r^{(k)}$ are order-$k$ B-splines defined on a knot grid $\Xi = \{t_j\}$,
with minimum knot spacing $h_{\min} := \min_j (t_{j+1} - t_j) > 0$.
Assume that $\sup_x |b'(x)| \le L_b$. Then $\phi$ is Lipschitz continuous, and its Lipschitz constant satisfies $\operatorname{Lip}(\phi)
\le |w_b|\, L_b
+ |w_s|\, \frac{2}{h_{\min}}\, \|c\|_1$.
\label{lemm}
\end{lemma}
The complete proof is in the appendix \ref{theosupp}. We emphasize that it is \emph{not} introduced to increase the spatial expressivity of the neural operator. Instead, KAN parameterizes a low-dimensional \emph{global modulation} pathway whose sensitivity can be explicitly controlled. Lemma \ref{lemm} provides a bounding and it provides an inspectable mechanism for regularizing global conditioning (e.g., via $\|c\|_{1}$ or knot spacing) and ensures that the global token varies smoothly with the input. Coupled with our single-query attention formulation, the resulting modulation is spatially uniform  (broadcast across $x$) and therefore cannot introduce new spatial degrees of freedom. Instead, it acts as a stable, Lipschitz-controlled controller that adaptively adjusts the spectral trunk across regimes, improving long-horizon rollout stability.

Then we further showed that KAN gives smooth, Lipschitz-controlled global modulation.

\begin{theorem}[Smooth, Lipschitz-controlled global modulation]
\label{theorem}
Assume Bounded query/key, Lipschitz pieces and KAN spline regularity. Then the global modulation map
$m : L^{\infty}(\mathcal{D}) \;\to\; \mathbb{R}^{d_m}$
is Lipschitz continuous, satisfying $\| m(u) - m(\tilde{u}) \|_2
\;\le\;
L_m \, \| u - \tilde{u} \|_{L^\infty}$,
where $L_m$ is an explicit constant depending on
$L_p,\; L_{\mathrm{KAN}},\; L_k,\; L_v,\; M_q,\; M_k,\; |\mathcal{D}|$,
as well as the exponential ``temperature'' scaling $1/\sqrt{d_k}$. In particular, decreasing $L_{\mathrm{KAN}}$ (e.g., via spline derivative control) directly reduces $L_m$.

\end{theorem}
The complete proof is in the appendix and  Theorem \ref{theorem} showed that KAN is a smooth global controller: it produces a low-dimensional token that changes the global modulation continuously with bounded sensitivity, which is beneficial for stable autoregressive PDE rollouts.

\begin{proposition}
(Informal)
The discrete attention output $A_h$ converges to the continuous output $A$ as $h \to 0$, and for sufficiently small $h$, $\| A_h - A \| \le \frac{\sqrt{d}\, h}{Z} \big( L_g + \|A\| L_w \big)$.
In particular, $\| A_h - A \| = O(h)$, as the grid is refined.
\label{prof}
\end{proposition}
$L_g$
 measures how fast the integrand changes in space, and $L_w$ measures how fast the attention weights change in space. The discrete softmax attention aggregation is a consistent quadrature approximation of a continuous, normalized kernel integral operator.
Proposition \ref{prof} shows as the mesh refines, the discrete GML converges to a resolution-independent continuous operator, with a provable error rate.

\section{Experiments}
This section presents the experimental evaluation of our method. We first describe the datasets and implementation details, followed by quantitative comparisons with representative baselines and ablation studies.

\textbf{Dataset, Baselines and Training Details.} We evaluate our approach on a variety of different PDEs.
In auto regressive training, we choose from PDEBench, which covers 1D CNS, 1D Diffusion Reaction, 2D Darcy flow, and 2D Shallow water. In addition to PDEBench \cite{takamoto2022pdebench}, we use both 2D climate modelling and Shallow Water Equations from \cite{kissas2022learning}. Detailed descriptions of the dataset and preprocessing procedures are provided in the Appendix \ref{dataset}. We compared our baseline with the following three categories of baselines. (1) FNO and it's varient: FNO \cite{li2020fourier}, FFNO \cite{tran2021factorized} (2) Transformer-based method: Transovler \cite{wu2024transolver}, CViT \cite{wang2024cvit} (3) KAN based method: DeepOKAN \cite{abueidda2025deepokan}, (4)Other method: UNet \cite{ronneberger2015u}, LNO \cite{cao2024laplace}, Local FNO\citep{liu2024neural}. All methods were trained in a single NVIDIA A100 40GB GPU. 

\textbf{Main Results.}
We conducted a comprehensive comparison of our method with representative baselines. The results are summarised in Table~\ref{tab:main_rmse}, where we report the RMSE. Also, Table ~\ref{swrl2} was reported in Relative $L^2$.

\begin{figure}[htb]
    \centering
    \includegraphics[width=\columnwidth]{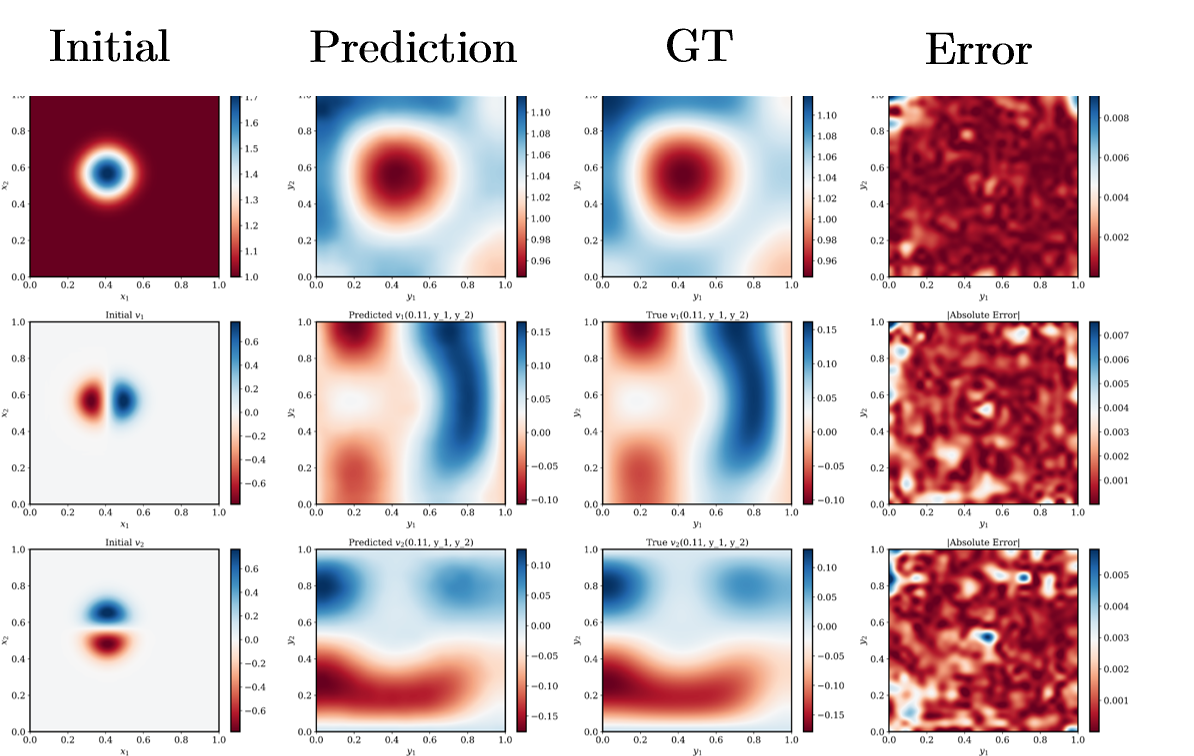}
    \caption{Qualitative comparison on the Shallow Water benchmark. Shown are the initial conditions, the predicted and ground-truth solutions, and the absolute error at 
    t=0.16s.}
    \label{fig:2dsw}
\end{figure}

Overall, our method achieves the best performance across all five benchmarks in Table~\ref{tab:main_rmse}, demonstrating consistent accuracy gains over strong recent baselines. In particular, on the 1D compressible Navier--Stokes (CNS) task, our approach reduces the RMSE from $0.29493$ (FFNO) to $0.14544$, corresponding to a $50.7\%$ relative improvement. For the 1D diffusion--reaction system, we obtain an RMSE of $0.00397$, improving upon the previous best result ($0.00686$) by $42.1\%$. On the 2D Darcy flow benchmark, where multiple methods already reach low errors, our model still provides a consistent gain, achieving $0.01260$ compared to $0.01315$. Notably, our largest improvements are observed on challenging spatio-temporal prediction tasks. On the 2D shallow-water benchmark, our method reaches $0.0024$ RMSE, yielding a $45.2\%$ reduction relative to the best competing baseline. Similarly, for 2D climate modelling, our approach attains $0.0096119$ RMSE, outperforming the strongest baseline ($0.01508$) by $36.3\%$. These results indicate that our method is not only competitive in the high-accuracy regime (e.g., Darcy flow), but also particularly effective at capturing complex dynamics in time-dependent systems. More numerical results can be foudn in Appendix \ref{exp}.

Table~\ref{swrl2} reports the other shallow water equations. Overall, our method achieves the lowest error across all components, consistently outperforming all other baselines. In particular, Our mehtod attains an error of $6.85\times10^{-4}$ on $\rho$, improving upon the strongest baseline FNO by $\approx 40.4\%$. Similar gains are observed for the velocity fields $u$ and $v$. We achieves the best performance on $v$ with $2.04\times10^{-2}$, corresponding to a $\approx 29.2\%$ reduction compared to the strongest baseline on $v$ (DeepOKAN, $2.88\times10^{-2}$). 

Figure~\ref{fig:fu1} compares qualitative results on the 1D CNS and 2D Darcy flow tasks. While baseline methods capture the coarse dynamics, they exhibit smoothing or spatial artifacts. Our method produces predictions closest to the ground truth, preserving sharper structures and physically consistent spatial patterns. Figure~\ref{fig:2dsw} shows a representative 2D shallow water result, where predictions at $t=0.11$ accurately match the ground-truth height and velocity fields, capturing the global flow structure with smooth spatial variations. More qualitative results can be found in Appendix \ref{vis}.

\begin{table}[t!]
\centering
\caption{Ablation study on architecture design, token construction, and resolution scaling. Best results are highlighted in \colorbox[HTML]{BBFFBB}{green} }
\label{tab:ablation_amfno}
\resizebox{0.9\linewidth}{!}{
\begin{tabular}{l c c c}
\toprule
\textbf{Method} 
& \textbf{Error $64^2$ } 
& $128^2$ 
& $32^2$ \\
\midrule

\multicolumn{4}{l}{\textbf{(a) Architecture Ablation on 2D Darcy Flow}} \\
\midrule
\textbf{AMFNO (baseline)} 
& 0.01756 & -- & -- \\
\quad + KAN (no GML) 
& \Sam{0.01327} & -- & -- \\
\quad + KAN + GML (ours) 
& \cellcolor[HTML]{BBFFBB}\textbf{0.01260} & -- & -- \\

\midrule
\multicolumn{4}{l}{\textbf{(b) Global Token Conditioning on 2D Darcyflow}} \\
\midrule
MLP 
& 0.01639 & -- & -- \\
Data (center) 
& \Sam{0.01616} & -- & -- \\
Data (mean) 
& 0.01693 & -- & -- \\
\textbf{KAN (Ours)}
& \cellcolor[HTML]{BBFFBB}\textbf{\textbf{0.0126}} & -- & -- \\

\midrule
\multicolumn{4}{l}{\textbf{(c) Super Resolution on 2D Shallow Water}} \\
\midrule
FNO 
& 0.01448 & 0.14369 & 0.17709 \\
Transformer 
& 0.00586 & \Sam{0.08827} & 0.09834 \\
Local FNO 
& \Sam{0.00438} & 3527.224 & \Sam{0.04693} \\
\textbf{SpectraKAN (Ours)} 
& \cellcolor[HTML]{BBFFBB}\textbf{0.00240} & \cellcolor[HTML]{BBFFBB}\textbf{\textbf{0.00458}} & \cellcolor[HTML]{BBFFBB}\textbf{\textbf{0.0107}} \\

\bottomrule
\end{tabular}
}
\end{table}
\textbf{Ablation study.}
Table~\ref{tab:ablation_amfno} investigates the contribution of our key design choices and performance in resolution invariance, including
(i) overall architecture,
(ii) token construction module, and
(iii) robustness under resolution changes. 

\textbf{(a) Architecture ablation.}

 Starting from the AMFNO baseline, introducing the KAN-based global conditioning reduces the error from 0.01756 to 0.01327 (24.4\% relative improvement). Replacing the non-GML mixing with GML mixing further improves performance to 0.0126 (28.2\% improvement over the baseline). These ablations indicate a clear separation of roles: KAN tokens provide global, non-stationary modulation, while GML mixing enhances long-range interaction across the field.

\textbf{(b) Global Token Conditioning.} We next isolate the effect of Global Token Conditioning by keeping all other components fixed. MLP tokens and simple statistic-based tokens (center/mean) achieve similar errors, indicating that these choices provide limited inductive bias for capturing the global state of the physical field. In contrast, KAN-based tokenization yields a substantial accuracy gain, reducing the error to 0.0126. This gap suggests that the functional structure of KAN tokens provides a more effective global representation—offering smoother, Lipschitz-regularized modulation—and thereby improves the operator ability to model global context and field-dependent dynamics .

\textbf{(c) Resolution Invarience.} We evaluate cross-resolution generalization by training at $64^2$ and evaluating both super-resolution at $128^2$ and down-sampling at $32^2$.
While FNO and Transovler are competitive at the training resolution, their performance deteriorates markedly under resolution shifts; for instance, FNO’s error increases by nearly an order of magnitude from $0.01448$ to $0.14369$ at $128^2$.
LocalFNO performs well at $64^2$ but becomes numerically unstable when upscaling, exploding to an error of $3527.224$ at $128^2$.
In contrast, our method achieves the lowest error at $64^2$ and remains robust to both scaling up and down, sustaining strong accuracy at $128^2$ ($0.00458$) and $32^2$ ($0.0107$).

\section{Conclusion}
We presented SpectraKAN, a neural operator that augments Fourier-based models with an explicit input-conditioned integral operator, changing the effective operator class from static spectral convolution to a globally conditioned, non-stationary operator.  We showed that the attention-based modulation admits a quadrature interpretation and converges to a resolution-independent continuous operator, while the KAN pathway provides a Lipschitz-controlled mechanism for stable global conditioning. Empirically, SpectraKAN achieves state-of-the-art performance across diverse PDE benchmarks, with particularly strong gains under spatio-temporal complexity and resolution shifts. These results suggest that conditioning spectral operators through global integral modulation is a principled direction for advancing neural operator design beyond static Fourier kernels.

\section*{Acknowledgments}
CWC is funded by the Swiss National Science Foundation (SNSF) under grant number 20HW-1 220785.   
CBS acknowledges support from the Philip Leverhulme Prize, the Royal Society Wolfson Fellowship, the EPSRC advanced career fellowship EP/V029428/1, EPSRC grants EP/S026045/1 and EP/T003553/1, EP/N014588/1, EP/T017961/1, the Wellcome Innovator Awards 215733/Z/19/Z and 221633/Z/20/Z, the European Union Horizon 2020 research and innovation
programme under the Marie Skodowska-Curie grant agreement No. 777826 NoMADS, the Cantab Capital Institute for the Mathematics of Information and the Alan Turing Institute. 
AIAR gratefully acknowledges the support of the Yau Mathematical Sciences Center, Tsinghua University. This work is also supported by the Tsinghua University Dushi Program.

\section*{Impact Statement}
This paper presents work whose goal is to advance the field of Machine Learning. There are many potential societal consequences of our work, none which we feel must be specifically highlighted here.

\nocite{langley00}

\bibliographystyle{icml2025}
\bibliography{main}

\begin{thebibliography}{38}
\providecommand{\natexlab}[1]{#1}
\providecommand{\url}[1]{\texttt{#1}}
\expandafter\ifx\csname urlstyle\endcsname\relax
  \providecommand{\doi}[1]{doi: #1}\else
  \providecommand{\doi}{doi: \begingroup \urlstyle{rm}\Url}\fi

\bibitem[Abueidda et~al.(2025)Abueidda, Pantidis, and Mobasher]{abueidda2025deepokan}
Abueidda, D.~W., Pantidis, P., and Mobasher, M.~E.
\newblock Deepokan: Deep operator network based on kolmogorov arnold networks for mechanics problems.
\newblock \emph{Computer Methods in Applied Mechanics and Engineering}, 436:\penalty0 117699, 2025.

\bibitem[Bryutkin et~al.(2024)Bryutkin, Huang, Deng, Yang, Sch{\"o}nlieb, and Aviles-Rivero]{bryutkin2024hamlet}
Bryutkin, A., Huang, J., Deng, Z., Yang, G., Sch{\"o}nlieb, C.-B., and Aviles-Rivero, A.
\newblock Hamlet: Graph transformer neural operator for partial differential equations.
\newblock \emph{arXiv preprint arXiv:2402.03541}, 2024.

\bibitem[Cao et~al.(2024)Cao, Goswami, and Karniadakis]{cao2024laplace}
Cao, Q., Goswami, S., and Karniadakis, G.~E.
\newblock Laplace neural operator for solving differential equations.
\newblock \emph{Nature Machine Intelligence}, 6\penalty0 (6):\penalty0 631--640, 2024.

\bibitem[Cao(2021)]{cao2021choose}
Cao, S.
\newblock Choose a transformer: Fourier or galerkin.
\newblock \emph{Advances in neural information processing systems}, 34:\penalty0 24924--24940, 2021.

\bibitem[Cheng et~al.(2025)Cheng, Huang, Zhang, Yang, Sch{\"o}nlieb, and Aviles-Rivero]{cheng2025mamba}
Cheng, C.-W., Huang, J., Zhang, Y., Yang, G., Sch{\"o}nlieb, C.-B., and Aviles-Rivero, A.~I.
\newblock Mamba neural operator: Who wins? transformers vs. state-space models for pdes.
\newblock \emph{Journal of Computational Physics}, pp.\  114567, 2025.

\bibitem[Dang et~al.()Dang, Nguyen, Nguyen, and Hy]{dangadaptfno}
Dang, H.~V., Nguyen, B.~D., Nguyen, P.~C., and Hy, T.-S.
\newblock Adaptfno: Adaptive fourier neural operator with dynamic spectral modes and multiscale learning for climate modeling.

\bibitem[Ding et~al.(2025)Ding, Zhang, Shi, Li, Gu, Xu, Xie, Zhao, Shi, and Liu]{ding2025physics}
Ding, S., Zhang, Z., Shi, G., Li, X., Gu, X., Xu, Y., Xie, H., Zhao, H., Shi, Y., and Liu, T.
\newblock Physics-informed neural operator learning for nonlinear grad-shafranov equation.
\newblock \emph{arXiv preprint arXiv:2511.19114}, 2025.

\bibitem[Guo \& Li(2024)Guo and Li]{guo2024mgfno}
Guo, Z.-H. and Li, H.-B.
\newblock Mgfno: Multi-grid architecture fourier neural operator for parametric partial differential equations.
\newblock \emph{arXiv preprint arXiv:2407.08615}, 2024.

\bibitem[Gupta et~al.(2021)Gupta, Xiao, and Bogdan]{gupta2021multiwavelet}
Gupta, G., Xiao, X., and Bogdan, P.
\newblock Multiwavelet-based operator learning for differential equations.
\newblock \emph{Advances in neural information processing systems}, 34:\penalty0 24048--24062, 2021.

\bibitem[Hao et~al.(2023)Hao, Wang, Su, Ying, Dong, Liu, Cheng, Song, and Zhu]{hao2023gnot}
Hao, Z., Wang, Z., Su, H., Ying, C., Dong, Y., Liu, S., Cheng, Z., Song, J., and Zhu, J.
\newblock Gnot: A general neural operator transformer for operator learning.
\newblock In \emph{International Conference on Machine Learning}, pp.\  12556--12569. PMLR, 2023.

\bibitem[Hu et~al.()Hu, Cao, Kawaguchi, and Karniadakis]{hu5149007deepomamba}
Hu, Z., Cao, Q., Kawaguchi, K., and Karniadakis, G.~E.
\newblock Deepomamba: State-space model for spatio-temporal pde neural operator learning.
\newblock \emph{Available at SSRN 5149007}.

\bibitem[Kissas et~al.(2022)Kissas, Seidman, Guilhoto, Preciado, Pappas, and Perdikaris]{kissas2022learning}
Kissas, G., Seidman, J.~H., Guilhoto, L.~F., Preciado, V.~M., Pappas, G.~J., and Perdikaris, P.
\newblock Learning operators with coupled attention.
\newblock \emph{Journal of Machine Learning Research}, 23\penalty0 (215):\penalty0 1--63, 2022.

\bibitem[Kossaifi et~al.(2023)Kossaifi, Kovachki, Azizzadenesheli, and Anandkumar]{kossaifi2023multi}
Kossaifi, J., Kovachki, N., Azizzadenesheli, K., and Anandkumar, A.
\newblock Multi-grid tensorized fourier neural operator for high-resolution pdes.
\newblock \emph{arXiv preprint arXiv:2310.00120}, 2023.

\bibitem[Kovachki et~al.(2023)Kovachki, Li, Liu, Azizzadenesheli, Bhattacharya, Stuart, and Anandkumar]{kovachki2023neural}
Kovachki, N., Li, Z., Liu, B., Azizzadenesheli, K., Bhattacharya, K., Stuart, A., and Anandkumar, A.
\newblock Neural operator: Learning maps between function spaces with applications to pdes.
\newblock \emph{Journal of Machine Learning Research}, 24\penalty0 (89):\penalty0 1--97, 2023.

\bibitem[Langley(2000)]{langley00}
Langley, P.
\newblock Crafting papers on machine learning.
\newblock In Langley, P. (ed.), \emph{Proceedings of the 17th International Conference on Machine Learning (ICML 2000)}, pp.\  1207--1216, Stanford, CA, 2000. Morgan Kaufmann.

\bibitem[Lee et~al.(2025)Lee, Liu, Yu, Wang, Jeong, Niu, and Zhang]{lee2025kano}
Lee, J., Liu, Z., Yu, X., Wang, Y., Jeong, H., Niu, M.~Y., and Zhang, Z.
\newblock Kano: Kolmogorov-arnold neural operator.
\newblock \emph{arXiv preprint arXiv:2509.16825}, 2025.

\bibitem[Li et~al.(2024{\natexlab{a}})Li, Yu, Xing, Kirby, Narayan, and Zhe]{li2024multi}
Li, S., Yu, X., Xing, W., Kirby, R., Narayan, A., and Zhe, S.
\newblock Multi-resolution active learning of fourier neural operators.
\newblock In \emph{International Conference on Artificial Intelligence and Statistics}, pp.\  2440--2448. PMLR, 2024{\natexlab{a}}.

\bibitem[Li et~al.(2020)Li, Kovachki, Azizzadenesheli, Liu, Bhattacharya, Stuart, and Anandkumar]{li2020fourier}
Li, Z., Kovachki, N., Azizzadenesheli, K., Liu, B., Bhattacharya, K., Stuart, A., and Anandkumar, A.
\newblock Fourier neural operator for parametric partial differential equations.
\newblock \emph{arXiv preprint arXiv:2010.08895}, 2020.

\bibitem[Li et~al.(2022)Li, Meidani, and Farimani]{li2022transformer}
Li, Z., Meidani, K., and Farimani, A.~B.
\newblock Transformer for partial differential equations' operator learning.
\newblock \emph{arXiv preprint arXiv:2205.13671}, 2022.

\bibitem[Li et~al.(2023)Li, Peng, Yuan, and Wang]{li2023long}
Li, Z., Peng, W., Yuan, Z., and Wang, J.
\newblock Long-term predictions of turbulence by implicit u-net enhanced fourier neural operator.
\newblock \emph{Physics of Fluids}, 35\penalty0 (7), 2023.

\bibitem[Li et~al.(2024{\natexlab{b}})Li, Zheng, Kovachki, Jin, Chen, Liu, Azizzadenesheli, and Anandkumar]{li2024physics}
Li, Z., Zheng, H., Kovachki, N., Jin, D., Chen, H., Liu, B., Azizzadenesheli, K., and Anandkumar, A.
\newblock Physics-informed neural operator for learning partial differential equations.
\newblock \emph{ACM/IMS Journal of Data Science}, 1\penalty0 (3):\penalty0 1--27, 2024{\natexlab{b}}.

\bibitem[Liu et~al.(2024)Liu, Wang, Vaidya, Ruehle, Halverson, Solja{\v{c}}i{\'c}, Hou, and Tegmark]{liu2024kan}
Liu, Z., Wang, Y., Vaidya, S., Ruehle, F., Halverson, J., Solja{\v{c}}i{\'c}, M., Hou, T.~Y., and Tegmark, M.
\newblock Kan: Kolmogorov-arnold networks.
\newblock \emph{arXiv preprint arXiv:2404.19756}, 2024.

\bibitem[Liu-Schiaffini et~al.(2024)Liu-Schiaffini, Berner, Bonev, Kurth, Azizzadenesheli, and Anandkumar]{liu2024neural}
Liu-Schiaffini, M., Berner, J., Bonev, B., Kurth, T., Azizzadenesheli, K., and Anandkumar, A.
\newblock Neural operators with localized integral and differential kernels.
\newblock \emph{arXiv preprint arXiv:2402.16845}, 2024.

\bibitem[Lu et~al.(2019)Lu, Jin, and Karniadakis]{lu2019deeponet}
Lu, L., Jin, P., and Karniadakis, G.~E.
\newblock Deeponet: Learning nonlinear operators for identifying differential equations based on the universal approximation theorem of operators.
\newblock \emph{arXiv preprint arXiv:1910.03193}, 2019.

\bibitem[Rahman et~al.(2022)Rahman, Ross, and Azizzadenesheli]{rahman2022u}
Rahman, M.~A., Ross, Z.~E., and Azizzadenesheli, K.
\newblock U-no: U-shaped neural operators.
\newblock \emph{arXiv preprint arXiv:2204.11127}, 2022.

\bibitem[Ronneberger et~al.(2015)Ronneberger, Fischer, and Brox]{ronneberger2015u}
Ronneberger, O., Fischer, P., and Brox, T.
\newblock U-net: Convolutional networks for biomedical image segmentation.
\newblock In \emph{International Conference on Medical image computing and computer-assisted intervention}, pp.\  234--241. Springer, 2015.

\bibitem[Roub{\'\i}{\v{c}}ek(2005)]{roubivcek2005nonlinear}
Roub{\'\i}{\v{c}}ek, T.
\newblock \emph{Nonlinear partial differential equations with applications}.
\newblock Springer, 2005.

\bibitem[Roy et~al.(2025)Roy, Bahmani, Kevrekidis, and Shields]{roy2025physics}
Roy, S., Bahmani, B., Kevrekidis, I.~G., and Shields, M.~D.
\newblock A physics-informed multi-resolution neural operator.
\newblock \emph{arXiv preprint arXiv:2510.23810}, 2025.

\bibitem[Takamoto et~al.(2022)Takamoto, Praditia, Leiteritz, MacKinlay, Alesiani, Pfl{\"u}ger, and Niepert]{takamoto2022pdebench}
Takamoto, M., Praditia, T., Leiteritz, R., MacKinlay, D., Alesiani, F., Pfl{\"u}ger, D., and Niepert, M.
\newblock Pdebench: An extensive benchmark for scientific machine learning.
\newblock \emph{Advances in Neural Information Processing Systems}, 35:\penalty0 1596--1611, 2022.

\bibitem[Tran et~al.(2021)Tran, Mathews, Xie, and Ong]{tran2021factorized}
Tran, A., Mathews, A., Xie, L., and Ong, C.~S.
\newblock Factorized fourier neural operators.
\newblock \emph{arXiv preprint arXiv:2111.13802}, 2021.

\bibitem[Vaswani et~al.(2017)Vaswani, Shazeer, Parmar, Uszkoreit, Jones, Gomez, Kaiser, and Polosukhin]{vaswani2017attention}
Vaswani, A., Shazeer, N., Parmar, N., Uszkoreit, J., Jones, L., Gomez, A.~N., Kaiser, {\L}., and Polosukhin, I.
\newblock Attention is all you need.
\newblock \emph{Advances in neural information processing systems}, 30, 2017.

\bibitem[Wang et~al.(2024)Wang, Seidman, Sankaran, Wang, Pappas, and Perdikaris]{wang2024cvit}
Wang, S., Seidman, J.~H., Sankaran, S., Wang, H., Pappas, G.~J., and Perdikaris, P.
\newblock Cvit: Continuous vision transformer for operator learning.
\newblock \emph{arXiv preprint arXiv:2405.13998}, 2024.

\bibitem[Wang et~al.(2025)Wang, Yang, Yuan, Li, Peng, and Wang]{wang2025hybrid}
Wang, Y., Yang, H., Yuan, Z., Li, Z., Peng, W., and Wang, J.
\newblock A hybrid u-net and fourier neural operator framework for the fast prediction of turbulent flows with mixed periodic and non-periodic boundary conditions.
\newblock \emph{arXiv preprint arXiv:2504.13126}, 2025.

\bibitem[Wen et~al.(2022)Wen, Li, Azizzadenesheli, Anandkumar, and Benson]{wen2022u}
Wen, G., Li, Z., Azizzadenesheli, K., Anandkumar, A., and Benson, S.~M.
\newblock U-fno—an enhanced fourier neural operator-based deep-learning model for multiphase flow.
\newblock \emph{Advances in Water Resources}, 163:\penalty0 104180, 2022.

\bibitem[Wu et~al.(2024)Wu, Luo, Wang, Wang, and Long]{wu2024transolver}
Wu, H., Luo, H., Wang, H., Wang, J., and Long, M.
\newblock Transolver: A fast transformer solver for pdes on general geometries.
\newblock \emph{arXiv preprint arXiv:2402.02366}, 2024.

\bibitem[Yeo et~al.(2024)Yeo, Nguyen, Le, and Mishra]{yeo2024kan}
Yeo, S., Nguyen, P.~A., Le, A.~N., and Mishra, S.
\newblock Kan-pdes: A novel approach to solving partial differential equations using kolmogorov-arnold networks—enhanced accuracy and efficiency.
\newblock In \emph{International Conference on Electrical and Electronics Engineering}, pp.\  43--62. Springer, 2024.

\bibitem[You et~al.(2024)You, Xu, and Cai]{you2024mscalefno}
You, Z., Xu, Z., and Cai, W.
\newblock Mscalefno: Multi-scale fourier neural operator learning for oscillatory function spaces.
\newblock \emph{arXiv preprint arXiv:2412.20183}, 2024.

\bibitem[Zhang et~al.(2025)Zhang, Wang, Zhang, Shen, and Zhang]{zhang2025physics}
Zhang, Z., Wang, Q., Zhang, Y., Shen, T., and Zhang, W.
\newblock Physics-informed neural networks with hybrid kolmogorov-arnold network and augmented lagrangian function for solving partial differential equations.
\newblock \emph{Scientific Reports}, 15\penalty0 (1):\penalty0 10523, 2025.

\end{thebibliography}

\newpage
\appendix
\onecolumn

\section*{Appendix Roadmap}
\label{sec:appendix-roadmap}

This appendix is organized to support the main paper in a bottom-up manner:
we first provide complete theoretical proofs for the claims in \S3.5, then
we report additional experimental results and qualitative visualizations.

\vspace{0.25em}
\noindent\textbf{A.\ Theoretical understanding (Proofs and operator view).}
\begin{itemize}
  \item \textbf{A.1 Proof of Lemma 3.1 (Lemma A.1).}
  We prove the \emph{edge Lipschitz bound} for spline-based KAN edge functions,
  establishing explicit control of $\operatorname{Lip}(\phi)$ via knot spacing and spline coefficients.
  \item \textbf{A.2 Proof of Theorem 3.2 (Theorem A.2).}
  We prove that the \emph{global modulation map} induced by the KAN token is Lipschitz,
  and show how the constant depends on the KAN pathway only through $L_{\mathrm{KAN}}$,
  making KAN a controllable global smoothness controller.
  \item \textbf{A.3 Proof of Proposition 3.3 (Proposition A.3).}
  We show that the discrete single-query attention used in GML is a consistent quadrature
  approximation of a continuous normalized integral operator, with $\|A_h-A\|=O(h)$ under mesh refinement.
  \item \textbf{A.4 Neural-operator interpretation (Theorem A.4).}
  We formalize the attention kernel $\kappa_u(x,y)$ and show that the GML module defines a
  (nonlinear) integral-kernel neural-operator layer; discretization yields standard softmax attention.
\end{itemize}

\vspace{0.25em}
\noindent\textbf{B.\ Additional Experiment results (More metrics and visual evidence).}
\begin{itemize}
  \item \ref{dataset} \textbf{Dataset Description.}
  We provide additional details of the dataset.
  \item \ref{exp} \textbf{Additional PDEBench results.}
  We report extended metrics beyond the main RMSE table, including normalized RMSE,
  max RMS error, conserved-variable RMSE, boundary RMSE, and Fourier-space RMSE by frequency bands.
  \item \ref{vis} \textbf{Additional visualization results.}
  We provide supplementary qualitative comparisons (multiple time snapshots / panels)
  to highlight stability and fine-scale fidelity across challenging spatio-temporal dynamics.
\end{itemize}

\section{Theoretical understanding}
\label{theosupp}
In this section, we first provide the proof of Lemma \ref{lemm}, then we provide the proof of Theorem \ref{theorem} and Proposition \ref{prof}.  Finally, we
showed that our method is a neural operator.

We first provide the proof of Lemma \ref{lemm}.
\begin{lemma}[Edge Lipschitz bound for spline KAN]
Consider an edge function of the form $\phi(x) = w_b\, b(x) + w_s \sum_{r} c_r\, B_r^{(k)}(x;\Xi)$,
where $b \in C^1$, and $B_r^{(k)}$ are order-$k$ B-splines defined on a knot grid $\Xi = \{t_j\}$,
with minimum knot spacing $h_{\min} := \min_j (t_{j+1} - t_j) > 0$.
Assume that $\sup_x |b'(x)| \le L_b$. Then $\phi$ is Lipschitz continuous, and its Lipschitz constant satisfies
\begin{equation}
\operatorname{Lip}(\phi)
\le |w_b|\, L_b
+ |w_s|\, \frac{2}{h_{\min}}\, \|c\|_1.
\end{equation}
\end{lemma}

\begin{proof}
B-splines satisfy the derivative identity
\[
\frac{d}{dx} B_r^{(k)}(x)
=
\frac{k}{t_{r+k} - t_r}
\left(
B_r^{(k-1)}(x) - B_{r+1}^{(k-1)}(x)
\right).
\]

Since $B_r^{(k-1)} \in [0,1]$ and $t_{r+k} - t_r \ge k h_{\min}$, we get
\[
\left|
\frac{d}{dx} B_r^{(k)}(x)
\right|
\le
\frac{k}{k h_{\min}}
\left(
|B_r^{(k-1)}(x)| + |B_{r+1}^{(k-1)}(x)|
\right)
\le
\frac{2}{h_{\min}}.
\]

Therefore,
\[
|\phi'(x)|
\le
|w_b||b'(x)|
+
|w_s|
\sum_r |c_r|
\left|
\frac{d}{dx} B_r^{(k)}(x)
\right|
\le
|w_b| L_b
+
|w_s| \frac{2}{h_{\min}} \|c\|_1 .
\]

By the mean value theorem, $\mathrm{Lip}(\phi) \le \sup_x |\phi'(x)|$, giving the result.
\end{proof}

Then we provide the proof of Theorem \ref{theorem}.
\begin{theorem}
    Assume the following:

\begin{enumerate}
\item \textbf{(Bounded query/key)}
There exist constants $M_q, M_k$ such that for all admissible $u$,
\[
\| q(\Phi_{\mathrm{KAN}}(P(u))) \|_2 \le M_q,
\qquad
\| k(u) \|_\infty \le M_k .
\]

\item \textbf{(Lipschitz pieces)}
There exist constants $L_P, L_{\mathrm{KAN}}, L_k, L_v$ such that
\[
\| P(u) - P(\tilde u) \|_2 \le L_P \| u - \tilde u \|_\infty,
\]
\[
\| \Phi_{\mathrm{KAN}}(z) - \Phi_{\mathrm{KAN}}(\tilde z) \|_2
\le L_{\mathrm{KAN}} \| z - \tilde z \|_2,
\]
\[
\| k(u) - k(\tilde u) \|_\infty \le L_k \| u - \tilde u \|_\infty,
\qquad
\| v(u) - v(\tilde u) \|_\infty \le L_v \| u - \tilde u \|_\infty .
\]

\item \textbf{(KAN spline regularity)}
Each edge function in the KAN layer is a spline of order $k$
and is $C^1$ with uniformly bounded derivative
(this is satisfied by B-splines with nondegenerate knot spacing; proven below).
\end{enumerate}

Then the global modulation map $m : L^\infty(\mathcal{D}) \to \mathbb{R}^{d_t}$
is Lipschitz:
\[
\| m(u) - m(\tilde u) \|_2
\le
L_m \| u - \tilde u \|_\infty,
\]
where $L_m$ is an explicit constant depending on
\[
L_P, \; L_{\mathrm{KAN}}, \; L_k, \; L_v, \; M_q, \; M_k, \; |\mathcal{D}|,
\]
and the exponential ``temperature'' scaling $1 / \sqrt{d_a}$.
In particular, decreasing $L_{\mathrm{KAN}}$ (via spline derivative control)
directly decreases $L_m$.
\end{theorem}
\begin{proof}
    Lipschitz bound with KAN appearing only as a global smoothness controller.

Let
\[
s_u(y)
=
\frac{1}{\sqrt{d_a}}
\langle q_u, k_u(y) \rangle,
\qquad
q_u := q\!\left(\Phi_{\mathrm{KAN}}(P(u))\right),
\qquad
k_u(\cdot) := k(u)(\cdot),
\]
and define
\[
w_u(y) = e^{s_u(y)}, 
\qquad
Z_u = \int_{\mathcal D} w_u(y)\,dy,
\qquad
\pi_u(y) = \frac{w_u(y)}{Z_u}.
\]
Also define the numerator
\[
N_u
=
\int_{\mathcal D} w_u(y)\, v_u(y)\, dy,
\qquad
v_u(\cdot) := v(u)(\cdot).
\]
Then the global modulation map is
\[
m(u) = \frac{N_u}{Z_u}.
\]

\paragraph{Step 1: Uniform bounds on $s_u$, $w_u$, and $Z_u$.}
From Assumption~1,
\[
|s_u(y)|
\le
\frac{1}{\sqrt{d_a}}
\|q_u\|_2 \, \|k_u(y)\|_2
\le
\frac{M_q M_k}{\sqrt{d_a}}
=: M_s.
\]
Hence
\[
w_u(y) \in [e^{-M_s}, e^{M_s}],
\qquad
Z_u
=
\int_{\mathcal D} w_u(y)\,dy
\ge
|\mathcal D| e^{-M_s}
=: Z_{\min} > 0 .
\]

\paragraph{Step 2: Bounding $\|w_u - w_{\tilde u}\|_{L^\infty}$.}
On the interval $[-M_s, M_s]$, the exponential is Lipschitz with constant $e^{M_s}$:
\[
|e^a - e^b| \le e^{M_s} |a-b|,
\qquad
\forall a,b \in [-M_s, M_s].
\]
Therefore,
\[
\| w_u - w_{\tilde u} \|_{L^\infty}
\le
e^{M_s} \| s_u - s_{\tilde u} \|_{L^\infty}.
\]

\paragraph{Step 3: Bounding $\|s_u - s_{\tilde u}\|_{L^\infty}$ and exposing $L_{\mathrm{KAN}}$.}
For any $y$,
\[
s_u(y) - s_{\tilde u}(y)
=
\frac{1}{\sqrt{d_a}}
\Big(
\langle q_u - q_{\tilde u}, k_u(y) \rangle
+
\langle q_{\tilde u}, k_u(y) - k_{\tilde u}(y) \rangle
\Big).
\]
Thus,
\[
|s_u(y) - s_{\tilde u}(y)|
\le
\frac{1}{\sqrt{d_a}}
\Big(
\|q_u - q_{\tilde u}\|_2 \|k_u(y)\|_2
+
\|q_{\tilde u}\|_2 \|k_u(y) - k_{\tilde u}(y)\|_2
\Big).
\]
Taking the supremum over $y$ and using boundedness,
\[
\| s_u - s_{\tilde u} \|_{L^\infty}
\le
\frac{1}{\sqrt{d_a}}
\left(
M_k \|q_u - q_{\tilde u}\|_2
+
M_q \|k_u - k_{\tilde u}\|_{L^\infty}
\right).
\]

Now,
\[
\|q_u - q_{\tilde u}\|_2
\le
\mathrm{Lip}(q)
\|\Phi_{\mathrm{KAN}}(P(u)) - \Phi_{\mathrm{KAN}}(P(\tilde u))\|_2
\le
\mathrm{Lip}(q)\, L_{\mathrm{KAN}}\, L_P
\|u - \tilde u\|_{L^\infty},
\]
and
\[
\|k_u - k_{\tilde u}\|_{L^\infty}
\le
L_k \|u - \tilde u\|_{L^\infty}.
\]
Hence,
\[
\| s_u - s_{\tilde u} \|_{L^\infty}
\le
\frac{1}{\sqrt{d_a}}
\left(
M_k \mathrm{Lip}(q) L_{\mathrm{KAN}} L_P
+
M_q L_k
\right)
\|u - \tilde u\|_{L^\infty}.
\]

Combining with Step~2,
\[
\| w_u - w_{\tilde u} \|_{L^\infty}
\le
\underbrace{
\frac{e^{M_s}}{\sqrt{d_a}}
\left(
M_k \mathrm{Lip}(q) L_{\mathrm{KAN}} L_P
+
M_q L_k
\right)
}_{=: C_w}
\|u - \tilde u\|_{L^\infty}.
\]

\paragraph{Step 4: Bounding numerator and denominator differences.}
For the denominator,
\[
|Z_u - Z_{\tilde u}|
=
\left|
\int_{\mathcal D} (w_u - w_{\tilde u})\,dy
\right|
\le
|\mathcal D|\, C_w \|u - \tilde u\|_{L^\infty}.
\]

For the numerator,
\[
N_u - N_{\tilde u}
=
\int_{\mathcal D}
\Big(
(w_u - w_{\tilde u}) v_u
+
w_{\tilde u} (v_u - v_{\tilde u})
\Big) dy.
\]
Thus,
\[
\|N_u - N_{\tilde u}\|_2
\le
|\mathcal D|
\Big(
M_v C_w
+
e^{M_s} L_v
\Big)
\|u - \tilde u\|_{L^\infty},
\]
where $M_v := \sup_u \|v_u\|_{L^\infty}$.
Moreover,
\[
\|N_{\tilde u}\|_2
\le
|\mathcal D| e^{M_s} M_v.
\]

\paragraph{Step 5: Difference of ratios.}
Using
\[
\frac{N_u}{Z_u} - \frac{N_{\tilde u}}{Z_{\tilde u}}
=
\frac{N_u - N_{\tilde u}}{Z_u}
+
N_{\tilde u}
\left(
\frac{1}{Z_u} - \frac{1}{Z_{\tilde u}}
\right),
\qquad
\left|
\frac{1}{Z_u} - \frac{1}{Z_{\tilde u}}
\right|
=
\frac{|Z_u - Z_{\tilde u}|}{Z_u Z_{\tilde u}},
\]
and noting $Z_u, Z_{\tilde u} \ge Z_{\min}$, we obtain
\[
\| m(u) - m(\tilde u) \|_2
\le
\frac{1}{Z_{\min}} \|N_u - N_{\tilde u}\|_2
+
\frac{1}{Z_{\min}^2}
\|N_{\tilde u}\|_2
|Z_u - Z_{\tilde u}|.
\]
Substituting the bounds above and $Z_{\min} = |\mathcal D| e^{-M_s}$ yields a finite constant $L_m$ such that
\[
\| m(u) - m(\tilde u) \|_2
\le
L_m \|u - \tilde u\|_{L^\infty}.
\]

\end{proof}

\paragraph{Conclusion.}
Crucially, the Lipschitz constant $L_m$ depends on the token network only through
$L_{\mathrm{KAN}}$ (and the query Lipschitz constant). Therefore, regularizing
KAN spline derivatives directly regularizes the global modulation sensitivity.

We first define the Continuous (limit) GML operator: Let $D \subset \mathbb{R}^d$ be a bounded domain and let
$w : D \to \mathbb{R}_{+}$ be a nonnegative weight function.
Define the continuous, normalized GML attention output
$A \in \mathbb{R}^{d_v}$ by
\begin{equation}
\label{eq:continuous-gml}
A
\;=\;
\frac{\displaystyle \int_{D} w(y)\, V\!\left(f(y)\right)\, dy}
     {\displaystyle \int_{D} w(y)\, dy}.
\end{equation}
This expression is a normalized kernel integral, and can be interpreted as
the expectation of $V(f(y))$ under the probability density proportional to $w(y)$. Then we can define the Discrete GML operator: Let $D$ be partitioned into $n^d$ uniform cells of side length $h = 1/n$,
and let $\{y_j\}_{j=1}^{N}$ denote the cell centers, where $N = n^d$.
The discrete GML operator is defined as
\begin{equation}
\label{eq:discrete-gml}
A_h
\;=\;
\frac{\displaystyle \sum_{j=1}^{N} w(y_j)\, V\!\left(f(y_j)\right)}
     {\displaystyle \sum_{j=1}^{N} w(y_j)}.
\end{equation}

\begin{proposition}
Assume:
\begin{enumerate}
    \item $w$ is Lipschitz on $D$, with Lipschitz constant $L_w$, and bounded:
    \[
    0 < w(y) \le M_w .
    \]

    \item $g(y) := w(y)V(f(y))$ is Lipschitz on $D$ (vector-valued), with Lipschitz constant $L_g$ in the Euclidean norm:
    \[
    \| g(y) - g(y') \| \le L_g \| y - y' \|.
    \]

    \item The normalizing constant
    \[
    Z = \int_D w(y)\, dy
    \]
    is positive (always true since $w>0$; we keep it explicit).
\end{enumerate}

Then the discrete attention output $A_h$ converges to the continuous output $A$ as $h \to 0$, and for sufficiently small $h$,
\[
\| A_h - A \| \le \frac{\sqrt{d}\, h}{Z} \big( L_g + \|A\| L_w \big).
\]

In particular,
\[
\| A_h - A \| = O(h)
\]
as the grid is refined.
\end{proposition}

\begin{proof}
\textbf{Step 1: Riemann-sum error bound for Lipschitz functions}

Partition the domain $D = [0,1]^d$ into $N = n^d$ hypercubes $\{C_j\}_{j=1}^N$ of side length
$h = 1/n$. Let $y_j$ denote the center of $C_j$.
For any Lipschitz scalar function $\phi : D \to \mathbb{R}$ with Lipschitz constant $L_\phi$,

\begin{align}
\left|
\int_D \phi(y)\,dy
-
\sum_{j=1}^N \phi(y_j)\,|C_j|
\right|
&\le
\sum_{j=1}^N
\int_{C_j}
\left|
\phi(y) - \phi(y_j)
\right| dy .
\end{align}

Since $\phi$ is Lipschitz,
\[
|\phi(y) - \phi(y_j)| \le L_\phi \|y - y_j\|.
\]
In a cube of side length $h$, the maximum distance from the center satisfies
$\|y - y_j\| \le \frac{\sqrt{d}}{2} h$. Therefore, for all $y \in C_j$,
\[
|\phi(y) - \phi(y_j)| \le L_\phi \frac{\sqrt{d}}{2} h.
\]

Thus,
\begin{align}
\int_{C_j} |\phi(y) - \phi(y_j)| dy
&\le
L_\phi \frac{\sqrt{d}}{2} h \, |C_j|
=
L_\phi \frac{\sqrt{d}}{2} h^{d+1}.
\end{align}

Summing over all $j$ and using $\sum_j |C_j| = |D| = 1$ yields
\begin{equation}
\left|
\int_D \phi(y)\,dy
-
h^d \sum_{j=1}^N \phi(y_j)
\right|
\le
L_\phi \frac{\sqrt{d}}{2} h .
\label{star}
\end{equation}

The same argument applies componentwise to vector-valued Lipschitz functions
$\psi : D \to \mathbb{R}^m$, yielding
\begin{equation}
\left\|
\int_D \psi(y)\,dy
-
h^d \sum_{j=1}^N \psi(y_j)
\right\|
\le
L_\psi \frac{\sqrt{d}}{2} h .
\label{starstar}
\end{equation}

\textbf{Step 2: Apply the bound to the attention numerator and denominator}
Define the continuous quantities
\[
Z = \int_D w(y)\,dy,
\qquad
N = \int_D w(y) V(f(y))\,dy = \int_D g(y)\,dy,
\]
and their Riemann-sum approximations
\[
Z_h = h^d \sum_{j=1}^N w(y_j),
\qquad
N_h = h^d \sum_{j=1}^N g(y_j).
\]

Applying \ref{star} to $w$ and \ref{starstar} to $g$ gives
\begin{align}
|Z - Z_h| &\le L_w \frac{\sqrt{d}}{2} h,
\\
\|N - N_h\| &\le L_g \frac{\sqrt{d}}{2} h.
\end{align}

The discrete attention output is
\[
A_h
=
\frac{\sum_j w(y_j) V(f(y_j))}{\sum_j w(y_j)}
=
\frac{N_h}{Z_h},
\]
since the factors of $h^d$ cancel.
The continuous attention output is $A = N / Z$.

\textbf{Step 3: Bounding the difference of ratios}

We compute
\[
A_h - A
=
\frac{N_h}{Z_h} - \frac{N}{Z}
=
\frac{Z (N_h - N) + N (Z - Z_h)}{Z Z_h}.
\]

Taking norms,
\begin{equation}
\|A_h - A\|
\le
\frac{
Z \|N_h - N\| + \|N\| |Z - Z_h|
}{Z Z_h}.
\label{3}
\end{equation}

Since $w > 0$, we have $Z > 0$.
Assume $h$ is sufficiently small such that
\[
|Z - Z_h| \le \frac{Z}{2},
\]
which holds whenever
\begin{equation}
h \le \frac{Z}{L_w \sqrt{d}}.
\end{equation}

Then $Z_h \ge Z/2$ and
\begin{equation}
Z Z_h \ge \frac{Z^2}{2}.
\label{5}
\end{equation}

Substituting \ref{5} into \ref{3},
\begin{align}
\|A_h - A\|
&\le
\frac{2}{Z}
\left(
\|N_h - N\| + \frac{\|N\|}{Z} |Z - Z_h|
\right).
\end{align}

Since $A = N / Z$, we have $\|N\|/Z = \|A\|$.
Applying bounds (1)–(2) yields
\begin{align}
\|A_h - A\|
&\le
\frac{2}{Z}
\left(
L_g \frac{\sqrt{d}}{2} h
+
\|A\| L_w \frac{\sqrt{d}}{2} h
\right)
\\
&=
\frac{\sqrt{d}\, h}{Z}
\left(
L_g + \|A\| L_w
\right).
\end{align}

\end{proof}

The discrete attention operator converges to its continuous counterpart at rate
\[
\boxed{
\|A_h - A\| = O(h)
}
\]
as $h \to 0$, with constants depending on the Lipschitz bounds of $w$ and $g$
and the magnitude of the continuous attention output.

\begin{theorem}
Define the (single-query) attention kernel
\begin{equation}
\kappa_u(x,y)
=
\frac{
\exp\!\left(
\frac{\langle q(\Phi_{\mathrm{KAN}}(u)),\, k(y)\rangle}{\sqrt{d_a}}
\right)
}{
\int_{D}
\exp\!\left(
\frac{\langle q(\Phi_{\mathrm{KAN}}(u)),\, k(\xi)\rangle}{\sqrt{d_a}}
\right)
\, d\xi
},
\qquad x,y \in D,
\end{equation}
and the associated operator
\begin{equation}
\bigl(A_{\Phi_{\mathrm{KAN}}}(u)\bigr)(x)
=
\int_{D}
\kappa_u(x,y)\, V\bigl(f(u)(y)\bigr)\, dy,
\qquad x \in D .
\end{equation}

\noindent

Then the following statement holds:
(1)$A_{\Phi_{\mathrm{KAN}}}(u)$ is a \emph{nonlinear integral kernel operator} and therefore a special case of a \emph{neural operator layer}, where the kernel is allowed to depend on the entire input function.

(2)Discretizing the integral (e.g.\ via quadrature or a Riemann sum) yields the standard \emph{softmax attention} computation, i.e.\ a single-query cross-attention block.
\end{theorem}

\begin{proof}
\textbf{Step 1: Recall the neural-operator-layer template (nonlinear kernel form).}
In the neural operator framework, one admissible \emph{kernel integration} layer is the nonlinear integral kernel operator:
\begin{equation}
(\mathcal{K}(v))(x)
=
\int_{D}
\kappa\!\big(x,y, v(\Pi(x)), v(y)\big)\, v(y)\, d\nu(y),
\label{eq:NO_kernel}
\end{equation}
where the kernel $\kappa$ is allowed to depend on the input representation $v$, hence inducing nonlinearity.

Moreover, in the proof of Proposition 6 (transformers as neural operators) from \cite{kovachki2023neural}, the authors explicitly introduce the notation $\kappa_v$ to indicate that the kernel may depend on the \emph{entire function} $v$, not only on pointwise values. This is described as a ``straightforward generalization'' beyond the written form of \eqref{eq:NO_kernel}. This generalization is precisely what is required for kernels that depend on the global token
\[
t = \Phi_{\mathrm{KAN}}(u).
\]

\textbf{Step 2: Writing SKAN's GML attention as an integral kernel operator.}
In our model, the attention-based Global Modulation Layer (GML) is defined as the integral operator (Eq.~(16)):
\begin{equation}
(\mathcal{A}_t[u])(x)
=
\int_D
\kappa_t(x,y)\,
V\!\big(f(u)(y)\big)\, dy,
\label{eq:GML_integral}
\end{equation}
where $t = \Phi_{\mathrm{KAN}}(u)$ and $V$ is a linear value projection.

Since $V$ is linear, there exists a matrix $W_V$ such that
\[
V(z) = W_V z.
\]
Let
\[
v(\cdot) := f(u)(\cdot) \in \mathbb{R}^{C_{\mathrm{fno}}}.
\]
Then \eqref{eq:GML_integral} can be rewritten as
\begin{equation}
(\mathcal{A}_{\Phi_{\mathrm{KAN}}(u)}[u])(x)
=
\int_D
\kappa_u(x,y)\,
W_V\, v(y)\, dy.
\end{equation}
Define the matrix-valued kernel
\[
\widetilde{\kappa}_u(x,y) := \kappa_u(x,y)\, W_V.
\]
The operator becomes
\begin{equation}
(\mathcal{A}_{\Phi_{\mathrm{KAN}}(u)}[u])(x)
=
\int_D
\widetilde{\kappa}_u(x,y)\, v(y)\, dy,
\label{eq:GML_NO_form}
\end{equation}
which is of exactly the same structural form as a neural-operator kernel integration map $\mathcal{K}(\cdot)$.

\textbf{Step 3: Validity of the SKAN attention kernel.}
The single-query cross-attention kernel is given by
\begin{equation}
\kappa_u(x,y)
=
\frac{
\exp\!\left(
\langle q(\Phi_{\mathrm{KAN}}(u)),\, k(y) \rangle / \sqrt{d_a}
\right)
}{
\int_D
\exp\!\left(
\langle q(\Phi_{\mathrm{KAN}}(u)),\, k(\xi) \rangle / \sqrt{d_a}
\right)
\, d\xi
},
\label{eq:GML_kernel}
\end{equation}
with
\[
q(\Phi_{\mathrm{KAN}}(u)) = Q(\Phi_{\mathrm{KAN}}(u)), \qquad
k(y) = K(f(u)(y)) = K(v(y)).
\]
Thus, the kernel $\kappa_u$ depends on:
(i) the \emph{entire input field} $u$ via the global token $\Phi_{\mathrm{KAN}}(u)$,
(ii) the pointwise feature values $v(y)$ through the key map $K(v(y))$,
(iii) and is independent of $x$ in the single-query setting, up to broadcast/readout.
This dependence is exactly the generalized form $\kappa_v$ allowed in Proposition 6 of \cite{kovachki2023neural}. Hence, $\mathcal{A}_{\Phi_{\mathrm{KAN}}(u)}$ is a valid nonlinear neural-operator kernel map.

\textbf{Step 4: Discretization yields standard softmax attention.}
Let $\{y_j\}_{j=1}^N \subset D$ be quadrature points with weights $w_j$. Approximating the integrals by quadrature yields
\begin{align}
\int_D
\kappa_u(x,y)\, V(v(y))\, dy
&\approx
\sum_{j=1}^N
w_j\, \kappa_u(x,y_j)\, V(v(y_j)), \\
\int_D
\exp\!\left(
\langle q(t), k(\xi) \rangle / \sqrt{d_a}
\right)
d\xi
&\approx
\sum_{\ell=1}^N
w_\ell\,
\exp\!\left(
\langle q(t), k(y_\ell) \rangle / \sqrt{d_a}
\right).
\end{align}
For uniform weights, constants cancel, yielding discrete attention coefficients
\begin{equation}
\alpha_j(t)
=
\frac{
\exp\!\left(
\langle q(t), k(y_j) \rangle / \sqrt{d_a}
\right)
}{
\sum_{\ell=1}^N
\exp\!\left(
\langle q(t), k(y_\ell) \rangle / \sqrt{d_a}
\right)
}.
\end{equation}
Hence,
\begin{equation}
(\mathcal{A}_t[u])(x_i)
\approx
\sum_{j=1}^N
\alpha_j(t)\, V\!\big(f(u)(y_j)\big),
\end{equation}
which is exactly the standard $\mathrm{softmax}(QK^\top)V$ attention computation.

\textbf{Step 5: Pointwise fusion preserves the neural-operator structure.}
Finally, the attention output is broadcast and fused with local features via
\begin{equation}
\hat{u}(x)
=
R\!\big(f(u)(x),\, \mathcal{A}_t[u](x)\big),
\end{equation}
where $R$ is a pointwise ($1\times 1$ convolution / MLP) map. Such local readouts are standard Nemitski-type operators in neural-operator architectures. Therefore, the full GML block—nonlocal kernel integration followed by pointwise fusion—remains a valid neural-operator layer.

\end{proof}

Consequently, the transformer component of SKAN (the GML cross-attention together with broadcast fusion/readout) constitutes a \emph{neural-operator-style layer}.

\section{Additional Experiment results}
In this section, we present the dataset description, additional experimental results on PDEBench, and additional visualization results.


\subsection{Dataset}
\label{dataset}
For the PDEbench equations for 1D CNS, 1D Diffusion Reaction, 2D Darcy flow, and 2D Shallow water, we provide the following equation summary in a table \ref{PDEsummary} for the PDE. For the Shallow Water Equation and Climate modelling from \cite{kissas2022learning}.

For Climate modelling, the goal is to to approximate the relationship between surface air temperature and surface air pressure. It seek to learn a black-box operator that can be used to predict pressure directly from temperature. Concretely, we consider the mapping
\[
T(x) \;\longmapsto\; P(y),
\]
where \(x,y \in [-90,90] \times [0,360]\) denote latitude--longitude
coordinates. For a fixed day of the year, the (unknown) solution operator
maps the surface air temperature field to the corresponding surface air
pressure field.

In this setting, the input and output domains coincide, so that
\(X = Y\). Since both temperature and pressure are scalar fields, we have
\(d_u = d_s = 1\), and the spatial dimensions satisfy \(d_x = d_y = 2\),
corresponding to the two-dimensional latitude--longitude grid.
The operator learning problem can therefore be formulated as
\[
G : C(X,\mathbb{R}) \;\to\; C(X,\mathbb{R}),
\]
where \(G\) maps a continuous temperature field on \(X\) to a continuous
pressure field on the same domain.

For Shallow Water Equations, it describes the evolution of a thin fluid layer
when the horizontal length scales are much larger than the vertical one, so that the flow
can be treated as depth-averaged.

Let \(x=(x_1,x_2)\in(0,1)^2\) and \(t\in(0,1]\). The unknowns are the total fluid column
height \(\rho(t,x)\) and the depth-averaged horizontal velocity components
\(v_1(t,x)\) and \(v_2(t,x)\). With gravitational acceleration \(g\), the shallow-water
system can be written as
\begin{align}
\frac{\partial \rho}{\partial t}
+ \frac{\partial (\rho v_1)}{\partial x_1}
+ \frac{\partial (\rho v_2)}{\partial x_2}
&= 0, \\
\frac{\partial (\rho v_1)}{\partial t}
+ \frac{\partial}{\partial x_1}\!\left(\rho v_1^{2} + \frac{1}{2} g\rho^{2}\right)
+ \frac{\partial (\rho v_1 v_2)}{\partial x_2}
&= 0, \\
\frac{\partial (\rho v_2)}{\partial t}
+ \frac{\partial (\rho v_1 v_2)}{\partial x_1}
+ \frac{\partial}{\partial x_2}\!\left(\rho v_2^{2} + \frac{1}{2} g\rho^{2}\right)
&= 0 .
\end{align}

Equivalently, the system may be expressed in conservative form as
\[
\frac{\partial \tilde U}{\partial t}
+ \frac{\partial \tilde F(\tilde U)}{\partial x_1}
+ \frac{\partial \tilde G(\tilde U)}{\partial x_2}
= 0,
\]
with conserved variables and fluxes
\[
\tilde U =
\begin{pmatrix}
\rho \\ \rho v_1 \\ \rho v_2
\end{pmatrix},
\qquad
\tilde F =
\begin{pmatrix}
\rho v_1 \\
\rho v_1^2 + \frac{1}{2} g\rho^2 \\
\rho v_1 v_2
\end{pmatrix},
\qquad
\tilde G =
\begin{pmatrix}
\rho v_2 \\
\rho v_1 v_2 \\
\rho v_2^2 + \frac{1}{2} g\rho^2
\end{pmatrix}.
\]

Given an initial state, the corresponding solution operator maps the initial height and
velocity fields to the height and velocity fields at later times. Since the input and
output live on the same space--time domain, we take \(X=Y\) with
\(d_x=d_y=3\) (time plus two spatial coordinates) and \(d_u=d_s=3\) (three output channels).
The operator learning task is therefore
\[
G:\; C(X,\mathbb{R}^3)\;\to\; C(X,\mathbb{R}^3).
\]

It impose reflective (solid-wall) boundary conditions, i.e.\ no normal flow through the
boundary:
\[
v\cdot \hat n \;=\; v_1 n_{x_1} + v_2 n_{x_2} \;=\; 0
\qquad \text{on } \partial(0,1)^2,
\]
where \(\hat n = n_{x_1}\hat i + n_{x_2}\hat j\) denotes the unit outward normal.

To generate training data, sample random initial conditions corresponding to a
``falling droplet'' perturbation with zero initial velocities was used:
\begin{align}
\rho(0,x)
&= 1 + h \exp\!\left(-\frac{(x_1-\xi)^2 + (x_2-\zeta)^2}{w}\right), \\
v_1(0,x) &= 0, \qquad v_2(0,x) = 0,
\end{align}
where \(h\) controls the droplet amplitude, \(w\) its width, and \((\xi,\zeta)\) its
initial center. Because the velocities are identically zero at \(t=0\) for all realizations,
we instead define the operator input at the first time step \(t_0=\Delta t=0.002\,\mathrm{s}\),
using the state \((\rho(t_0,\cdot), v_1(t_0,\cdot), v_2(t_0,\cdot))\).

The random parameters are drawn from uniform distributions:
\[
h \sim \mathcal U(1.5,2.5),\qquad
w \sim \mathcal U(0.002,0.008),\qquad
\xi \sim \mathcal U(0.4,0.6),\qquad
\zeta \sim \mathcal U(0.4,0.6).
\]

The dataset was constructed by sampling these initial conditions on a \(32\times 32\) grid and
solving the forward problem with a Lax--Friedrichs scheme. This produces solutions on the
full \(32\times 32\) grid, which can then be sub-sampled to form training sets that use
only a fraction of the available spatial observations while still learning to predict the
full solution.

\begin{table}[t]
\centering
\small
\setlength{\tabcolsep}{6pt}
\renewcommand{\arraystretch}{1.25}
\begin{tabular}{@{}l p{0.58\linewidth} c c c@{}}
\toprule
\textbf{PDE} & \textbf{Governing equation(s)} & $\mathbf{N_s}$ & $\mathbf{N_t}$ & \textbf{\# samples} \\
\midrule

1D CNS &
\begin{minipage}[t]{\linewidth}\vspace{2pt}
(1D compressible Navier--Stokes)
\[
\partial_t\rho + \nabla\!\cdot(\rho v)=0,
\]
\[
\rho(\partial_t v + v\cdot \nabla v)= -\nabla p + \eta\,\partial_{xx}v + (\zeta+\eta/3)\,\nabla(\nabla \cdot v),
\]
\[
\partial_t\!\left(\epsilon + \frac{\rho v^2}{2}\right) + 
\nabla \cdot\!\left[\left(p+\epsilon+\frac{\rho v^2}{2}\right)v - v\cdot\sigma'\right]=0.
\]
\vspace{2pt}\end{minipage}
& $128$ & $20$ & $10000$ \\

1D Diffusion--Reaction &
\begin{minipage}[t]{\linewidth}\vspace{2pt}
\[
\partial_t u(t,x) - \nu\,\partial_{xx}u(t,x) - \rho\,u(t,x)\bigl(1-u(t,x)\bigr)=0.
\]
\vspace{2pt}\end{minipage}
& $256$ & $200$ & $10000$ \\

2D DarcyFlow &
\begin{minipage}[t]{\linewidth}\vspace{2pt}
(steady-state Darcy flow)
\[
-\nabla\!\cdot\!\bigl(a(x)\nabla u(x)\bigr)= f(x), \qquad x\in(0,1)^2,
\]
\[
u(x)=0, \qquad x\in\partial(0,1)^2.
\]
\vspace{2pt}\end{minipage}
& $64\times64$ & -- & $10000$ \\

2D Shallow Water &
\begin{minipage}[t]{\linewidth}\vspace{2pt}
\[
\partial_t h + \partial_x(hu) + \partial_y(hv)=0,
\]
\[
\partial_t(hu) + \partial_x\!\left(u^2h + \tfrac12 g_r h^2\right) + \partial_y(uvh)= -g_r h\,\partial_x b,
\]
\[
\partial_t(hv) + \partial_y\!\left(v^2h + \tfrac12 g_r h^2\right) + \partial_x(uvh)= -g_r h\,\partial_y b.
\]
\vspace{2pt}\end{minipage}
& $64\times64$ & $100$ & $1000$ \\

\bottomrule
\end{tabular}
\caption{PDEs and dataset sizes (\(N_s\): spatial resolution, \(N_t\): temporal resolution). For DarcyFlow, \(N_t\) is not applicable because the dataset is steady-state.}
\label{PDEsummary}
\end{table}

\subsection{Additional Experiment Results}
\label{exp}
In this section, we provide additional experiment result for the PDEBench, including normalized RMSE,  Maximum value of rms error,  RMSE of conserved
variables,  RMSE at boundaries,  RMSE in Fourier space - low,  RMSE in Fourier space - middle and  RMSE in Fourier space - high. 
\begin{table*}[t!]
  \centering
  \caption{ Performance comparison among our method and other representative baselines via different datasets in \Sam{normalized RMSE}. The best-performing results are highlighted in \colorbox[HTML]{BBFFBB}{green}, while the second-best results are indicated in \Sam{blue} colour.}
  \resizebox{\textwidth}{!}{
    \begin{tabular}{lcccc}
      \toprule
      \textbf{Method} 
      & \textbf{CNS (1D)} 
      & \textbf{Diffusion--Reaction (1D)} 
      & \textbf{Darcy Flow (2D)} 
      & \textbf{Shallow Water (2D)} \\
      \midrule

      U-Net (2015)
      & 0.8123
      & 0.07399
      & \cellcolor[HTML]{BBFFBB}\textbf{0.06215} 
      & 0.11552\\

      FNO (2021)
      &  0.23532
      &  0.05484
      &  0.13772
      &  0.01393 \\

      FFNO (2023)
      &  \cellcolor[HTML]{BBFFBB}\textbf{0.16448}
      &  0.02685
      &  0.13121
      & 0.01145 \\

      LNO (2024)
      &  0.72711
      & 0.09054
      &  0.21483
      &  0.03793\\

      Transolver (2024)
      & 0.35738
      & 0.01347
      & 0.10777 
      & 0.00565  \\

      LocalFNO (2024)
      & -- 
      & -- 
      & 0.10662
      & \Sam{0.00422} \\

      DeepOKAN (2025)
      & 2.77435
      & 0.08516 
      & 0.09196
      & 0.08282  \\

      CViT (2025)
      & 0.19233
      & \Sam{0.01263} 
      & 0.12313
      & 0.01133 \\

      \midrule
      \textbf{(Ours)} 
      & \Sam{0.17462} 
      & \cellcolor[HTML]{BBFFBB}\textbf{0.00814} 
      & \Sam{0.06304} 
      & \cellcolor[HTML]{BBFFBB}\textbf{0.00231} \\

      \midrule
      \textbf{Improvement (\%)} 
      & -6.16\% 
      & 35.55\% 
      & -1.43\% 
      & 45.26\% \\

      \bottomrule
    \end{tabular}
  }
\end{table*}

\begin{table*}[t!]
  \centering
  \caption{  Performance comparison among our method and other representative baselines via different datasets in \Sam{Maximum value of rms error}. The best-performing results are highlighted in \colorbox[HTML]{BBFFBB}{green}, while the second-best results are indicated in \Sam{blue} colour.}
  \resizebox{\textwidth}{!}{
    \begin{tabular}{lcccc}
      \toprule
      \textbf{Method} 
      & \textbf{CNS (1D)} 
      & \textbf{Diffusion--Reaction (1D)} 
      & \textbf{Darcy Flow (2D)} 
      & \textbf{Shallow Water (2D)} \\
      \midrule

      U-Net (2015)
      &  50.83448
      &  0.21527
      &  \cellcolor[HTML]{BBFFBB}\textbf{0.18478}
      &  0.75481\\

      FNO (2021)
      &  4.78869
      &  0.0521
      &  0.22714
      &  0.13625\\

      FFNO (2023)
      &  \Sam{4.55828}
      & \Sam{0.0304}
      &  0.25059
      &  0.12899\\

      LNO (2024)
      & 38.83067
      & 0.25961
      & 0.41448 
      & 0.26936 \\

      Transolver (2024)
      & 8.22922
      & 0.03871
      & 0.26815 
      & 0.12001  \\

      LocalFNO (2024)
      & -- 
      & -- 
      & 0.21291
      & \Sam{0.09241} \\

      DeepOKAN (2025)
      & 39.21762
      & 0.14834
      & 0.22246
      & 0.63804 \\

      CViT (2025)
      & 8.25447
      & 0.04548
      & 0.22511
      & 0.15085 \\

      \midrule
      \textbf{(Ours)} 
      & \cellcolor[HTML]{BBFFBB}\textbf{2.89683} 
      & \cellcolor[HTML]{BBFFBB}\textbf{0.02429} 
      & \Sam{0.19617} 
      & \cellcolor[HTML]{BBFFBB}\textbf{0.03033} \\

      \midrule
      \textbf{Improvement (\%)} 
      & 36.45\% 
      & 18.09\% 
      & -6.16\% 
      & 67.18\% \\

      \bottomrule
    \end{tabular}
  }
\end{table*}

\begin{table*}[t!]
  \centering
  \caption{ Performance comparison among our method and other representative baselines via different datasets in \Sam{RMSE of conserved variables}. The best-performing results are highlighted in \colorbox[HTML]{BBFFBB}{green}, while the second-best results are indicated in \Sam{blue} colour.}
  \resizebox{\textwidth}{!}{
    \begin{tabular}{lcccc}
      \toprule
      \textbf{Method} 
      & \textbf{CNS (1D)} 
      & \textbf{Diffusion--Reaction (1D)} 
      & \textbf{Darcy Flow (2D)} 
      & \textbf{Shallow Water (2D)} \\
      \midrule

      U-Net (2015)
      & 15.34809  
      & 0.05359 
      & 0.01437 
      & 0.03397 \\

      FNO (2021)
      & 0.0879
      & 0.03036 
      & 0.02368 
      & 0.00058 \\

      FFNO (2023)
      &  0.06288
      & 0.01511 
      & 0.01976 
      & 0.00069 \\

      LNO (2024)
      & 10.74895
      & 0.07057
      & 0.02867 
      & 0.00452 \\

      Transolver (2024)
      & 0.22427
      & \Sam{0.00963}
      & \Sam{0.0144} 
      & 0.00066  \\

      LocalFNO (2024)
      & -- 
      & -- 
      &  0.02042
      & \Sam{0.00038} \\

      DeepOKAN (2025)
      & 10.80045
      & 0.03845 
      & 0.0149
      & 0.01251 \\

      CViT (2025)
      & \Sam{0.05702}
      & 0.01131
      & 0.0166
      & 0.00144 \\

      \midrule
      \textbf{(Ours)} 
      & \cellcolor[HTML]{BBFFBB}\textbf{0.0433} 
      & \cellcolor[HTML]{BBFFBB}\textbf{0.00586} 
      & \cellcolor[HTML]{BBFFBB}\textbf{0.01183} 
      & \cellcolor[HTML]{BBFFBB}\textbf{0.00012} \\

      \midrule
      \textbf{Improvement (\%)} 
      & 24.37\% 
      & 39.15\% 
      & 17.85\% 
      & 68.42\% \\

      \bottomrule
    \end{tabular}
  }
 
\end{table*}\begin{table*}[t!]
  \centering
  \caption{ Performance comparison among our method and other representative baselines via different datasets in \Sam{RMSE at boundaries}. The best-performing results are highlighted in \colorbox[HTML]{BBFFBB}{green}, while the second-best results are indicated in \Sam{blue} colour.}
  \resizebox{\textwidth}{!}{
    \begin{tabular}{lcccc}
      \toprule
      \textbf{Method} 
      & \textbf{CNS (1D)} 
      & \textbf{Diffusion--Reaction (1D)} 
      & \textbf{Darcy Flow (2D)} 
      & \textbf{Shallow Water (2D)} \\
      \midrule

      U-Net (2015)
      &  11.72038
      &  0.03576
      &  \cellcolor[HTML]{BBFFBB}\textbf{0.00407}
      &  0.01159\\

      FNO (2021)
      & 0.2955
      & 0.02934 
      & 0.02069 
      & 0.00338 \\

      FFNO (2023)
      &  \Sam{0.22411}
      & 0.01408 
      & 0.0136 
      & 0.00364 \\

      LNO (2024)
      &  8.82148
      & 0.04704
      & 0.02373
      & 0.01552 \\

      Transolver (2024)
      & 0.62781
      & 0.00682
      & 0.01248 
      & 0.00371 \\

      LocalFNO (2024)
      & -- 
      & -- 
      & 0.00772
      & \Sam{0.00201} \\

      DeepOKAN (2025)
      & 9.05406
      & 0.0409 
      & 0.01221
      & 0.02932 \\

      CViT (2025)
      & 0.74478
      & \Sam{0.00675}
      & 0.01503
      & 0.00572 \\

      \midrule
      \textbf{(Ours)} 
      & \cellcolor[HTML]{BBFFBB}\textbf{0.11007} 
      & \cellcolor[HTML]{BBFFBB}\textbf{0.00386} 
      & \Sam{0.00581} 
      & \cellcolor[HTML]{BBFFBB}\textbf{0.00054} \\

      \midrule
      \textbf{Improvement (\%)} 
      & 50.89\% 
      & 45.48\% 
      & -42.75\% 
      & 73.13\% \\

      \bottomrule
    \end{tabular}
  }
  
\end{table*}

\begin{table*}[t!]
  \centering
  \caption{ Performance comparison among our method and other representative baselines via different datasets in \Sam{RMSE in Fourier space - low}. The best-performing results are highlighted in \colorbox[HTML]{BBFFBB}{green}, while the second-best results are indicated in \Sam{blue} colour.}
  \resizebox{\textwidth}{!}{
    \begin{tabular}{lcccc}
      \toprule
      \textbf{Method} 
      & \textbf{CNS (1D)} 
      & \textbf{Diffusion--Reaction (1D)} 
      & \textbf{Darcy Flow (2D)} 
      & \textbf{Shallow Water (2D)} \\
      \midrule

      U-Net (2015)
      &  4.42510
      & 0.01579993 
      & \Sam{0.00676289} 
      & 0.03163715 \\

      FNO (2021)
      & 0.11258
      & 0.0078190938 
      & 0.01132 
      & 0.00059789  \\

      FFNO (2023)
      & \Sam{0.09931}
      & 0.0040068841
      & 0.01056 
      & 0.00223955 \\

      LNO (2024)
      & 3.01964
      & 0.0183305
      & 0.01507362
      & 0.0067391  \\

      Transolver (2024)
      & 0.27045 
      & \Sam{0.0026418727} 
      & 0.00813 
      & 0.00073594 \\

      LocalFNO (2024)
      & -- 
      & -- 
      & 0.02896
      & \Sam{0.00041726} \\

      DeepOKAN (2025)
      & 3.09761
      & 0.01418 
      & 0.00787
      & 0.02002841 \\

      CViT (2025)
      & 0.22286
      & 0.00306786
      & 0.00936
      & 0.00136 \\

      \midrule
      \textbf{(Ours)} 
      & \cellcolor[HTML]{BBFFBB}\textbf{0.04996734} 
      & \cellcolor[HTML]{BBFFBB}\textbf{0.0016653806} 
      & \cellcolor[HTML]{BBFFBB}\textbf{0.00590517} 
      & \cellcolor[HTML]{BBFFBB}\textbf{0.00012509} \\

      \midrule
      \textbf{Improvement (\%)} 
      & 49.69\% 
      & 36.96\% 
      & 12.68\% 
      & 70.02\% \\

      \bottomrule
    \end{tabular}
  }
  
\end{table*}

\begin{table*}[t!]
  \centering
  \caption{ Performance comparison among our method and other representative baselines via different datasets in \Sam{RMSE in Fourier space - middle}. The best-performing results are highlighted in \colorbox[HTML]{BBFFBB}{green}, while the second-best results are indicated in \Sam{blue} colour.}
  \resizebox{\textwidth}{!}{
    \begin{tabular}{lcccc}
      \toprule
      \textbf{Method} 
      & \textbf{CNS (1D)} 
      & \textbf{Diffusion--Reaction (1D)} 
      & \textbf{Darcy Flow (2D)} 
      & \textbf{Shallow Water (2D)} \\
      \midrule

      U-Net (2015)
      & 0.18985
      & 0.00102379 
      & \cellcolor[HTML]{BBFFBB}\textbf{0.000643} 
      & 0.00874112 \\

      FNO (2021)
      &  0.04998
      &  0.000050198119
      &  0.00104
      &  0.0006694\\

      FFNO (2023)
      & \Sam{0.04275}
      & 0.000052739455 
      & 0.00152 
      & 0.00109296 \\

      LNO (2024)
      & 0.06662
      & 0.0000079677
      & 0.00211222 
      & 0.00487229  \\

      Transolver (2024)
      & 0.08406
      & 0.000080886944
      & 0.00128 
      & 0.00061484 \\

      LocalFNO (2024)
      & -- 
      & -- 
      & 0.00297
      & \Sam{0.00050069} \\

      DeepOKAN (2025)
      & 0.08463
      & 0.00321 
      & 0.00108
      & 0.00839461 \\

      CViT (2025)
      & 0.07991
      & \Sam{0.0000476469}
      & 0.00115
      & 0.00126 \\

      \midrule
      \textbf{(Ours)} 
      & \cellcolor[HTML]{BBFFBB}\textbf{0.0160807} 
      & \cellcolor[HTML]{BBFFBB}\textbf{0.000044599139} 
      & \Sam{0.00092448} 
      & \cellcolor[HTML]{BBFFBB}\textbf{0.00026044} \\

      \midrule
      \textbf{Improvement (\%)} 
      & 62.37\% 
      & 6.40\% 
      & -43.78\% 
      & 47.98\% \\

      \bottomrule
    \end{tabular}
  }
  
\end{table*}

\begin{table*}[t!]
  \centering
  \caption{ Performance comparison among our method and other representative baselines via different datasets in \Sam{RMSE in Fourier space - high}. The best-performing results are highlighted in \colorbox[HTML]{BBFFBB}{green}, while the second-best results are indicated in \Sam{blue} colour.}
  \resizebox{\textwidth}{!}{
    \begin{tabular}{lcccc}
      \toprule
      \textbf{Method} 
      & \textbf{CNS (1D)} 
      & \textbf{Diffusion--Reaction (1D)} 
      & \textbf{Darcy Flow (2D)} 
      & \textbf{Shallow Water (2D)} \\
      \midrule

      U-Net (2015)
      & 0.03155
      & 0.00022391 
      & \cellcolor[HTML]{BBFFBB}\textbf{0.00012363} 
      & 0.00196839 \\

      FNO (2021)
      & 0.00706
      & \Sam{0.0000031962038} 
      & \Sam{0.00026} 
      & 0.00147641 \\

      FFNO (2023)
      & 0.00686
      & 0.0000038141336 
      & 0.00053
      & 0.00042784 \\

      LNO (2024)
      & \Sam{0.00583}
      & \cellcolor[HTML]{BBFFBB}\textbf{0.0000000697743}
      &0.00075919
      & 0.00346834 \\

      Transolver (2024)
      & 0.01283
      & 0.000011153535 
      & 0.00041 
      & 0.00042714 \\

      LocalFNO (2024)
      & -- 
      & -- 
      & 0.00118
      & \Sam{0.00033259} \\

      DeepOKAN (2025)
      & 0.00893
      & 0.00019 
      & 0.00025
      & 0.00229325 \\

      CViT (2025)
      & 0.0156
      & 0.000024853
      & 0.0003
      & 0.00087 \\

      \midrule
      \textbf{(Ours)} 
      & \cellcolor[HTML]{BBFFBB}\textbf{0.00570552} 
      & 0.0000041244448 
      & 0.00030422 
      & \cellcolor[HTML]{BBFFBB}\textbf{0.00019776} \\

      \midrule
      \textbf{Improvement (\%)} 
      & 2.14\% 
      & -5,810.6\% 
      & -1.46\% 
      & 40.54\% \\

      \bottomrule
    \end{tabular}
  }
  
\end{table*}

\subsection{Additional Visulisation}
\label{vis}
In this section, we provide additional visulisation results for 2D Shallow water, Climate Modelling and Shallow Water Equations.
\begin{figure*}[t]
    \centering
    \includegraphics[width=\textwidth]{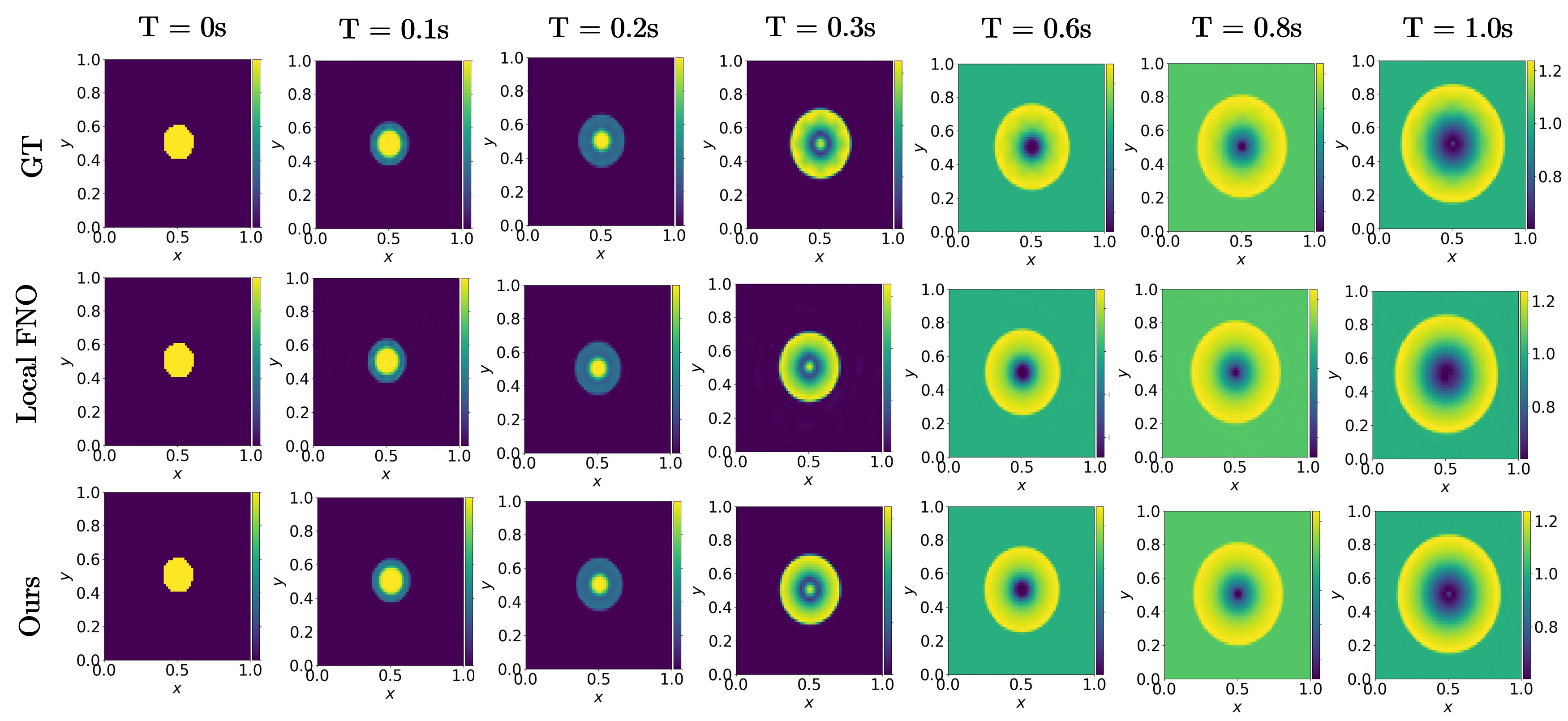}
    \caption{Qualitative comparison across Local FNO and ours in 2D shallow water.} 
    
\end{figure*}

\begin{figure*}[t]
    \centering
    \includegraphics[width=\textwidth]{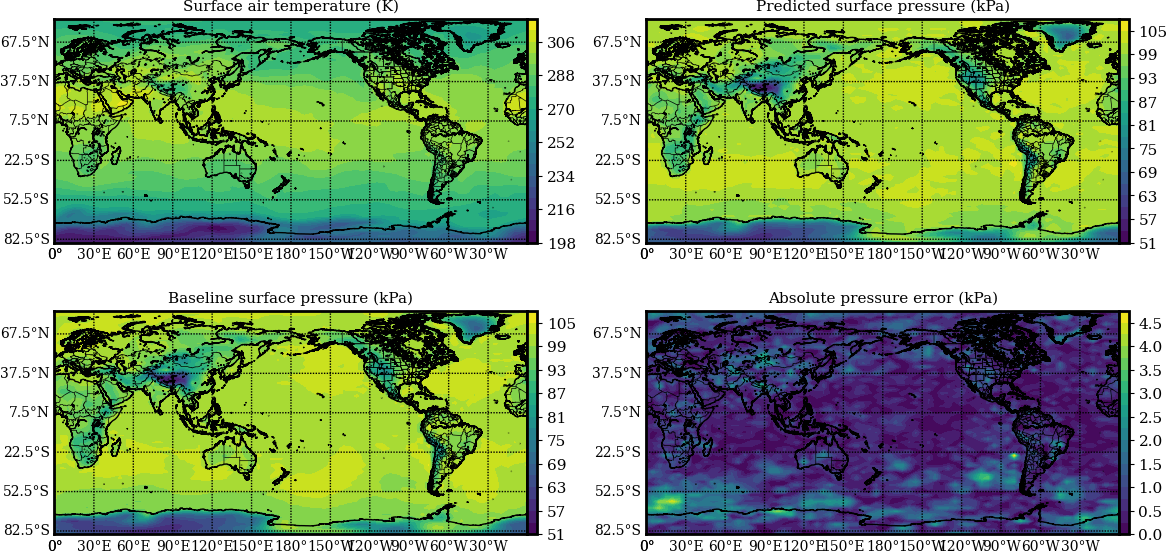}
    \caption{Full-resolution comparison on the climate modeling benchmark. Shown are the input temperature field, the model prediction, the ground truth, and the absolute error with respect to the reference solution..} 
    
\end{figure*}

\begin{figure*}[t]
    \centering
    \includegraphics[width=\textwidth]{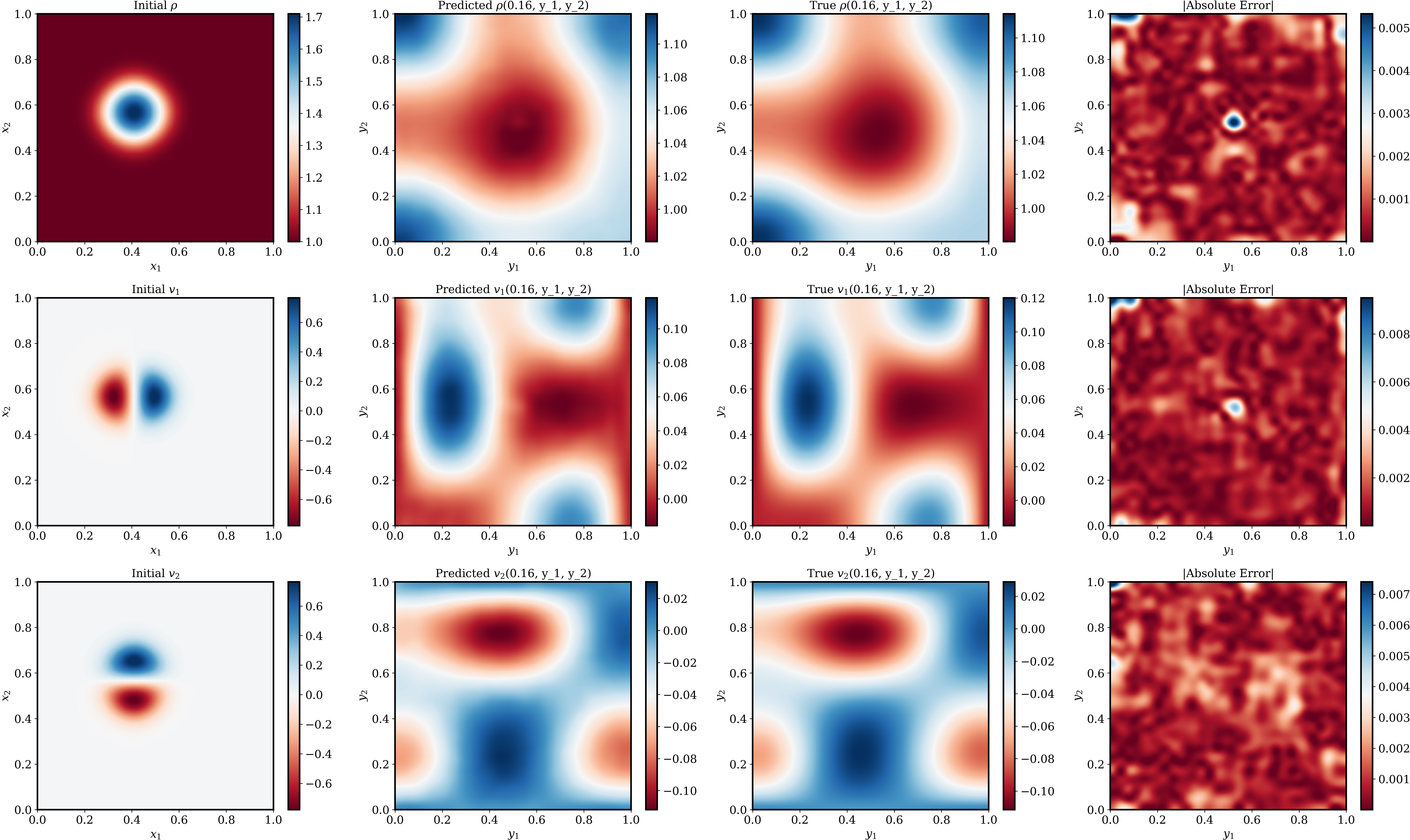}
    \caption{Qualitative comparison on the Shallow Water benchmark. Shown are the initial conditions, the predicted and ground-truth solutions, and the absolute error at 
    t=0.16s.} 
    
\end{figure*}

\begin{figure*}[t]
    \centering
    \includegraphics[width=\textwidth]{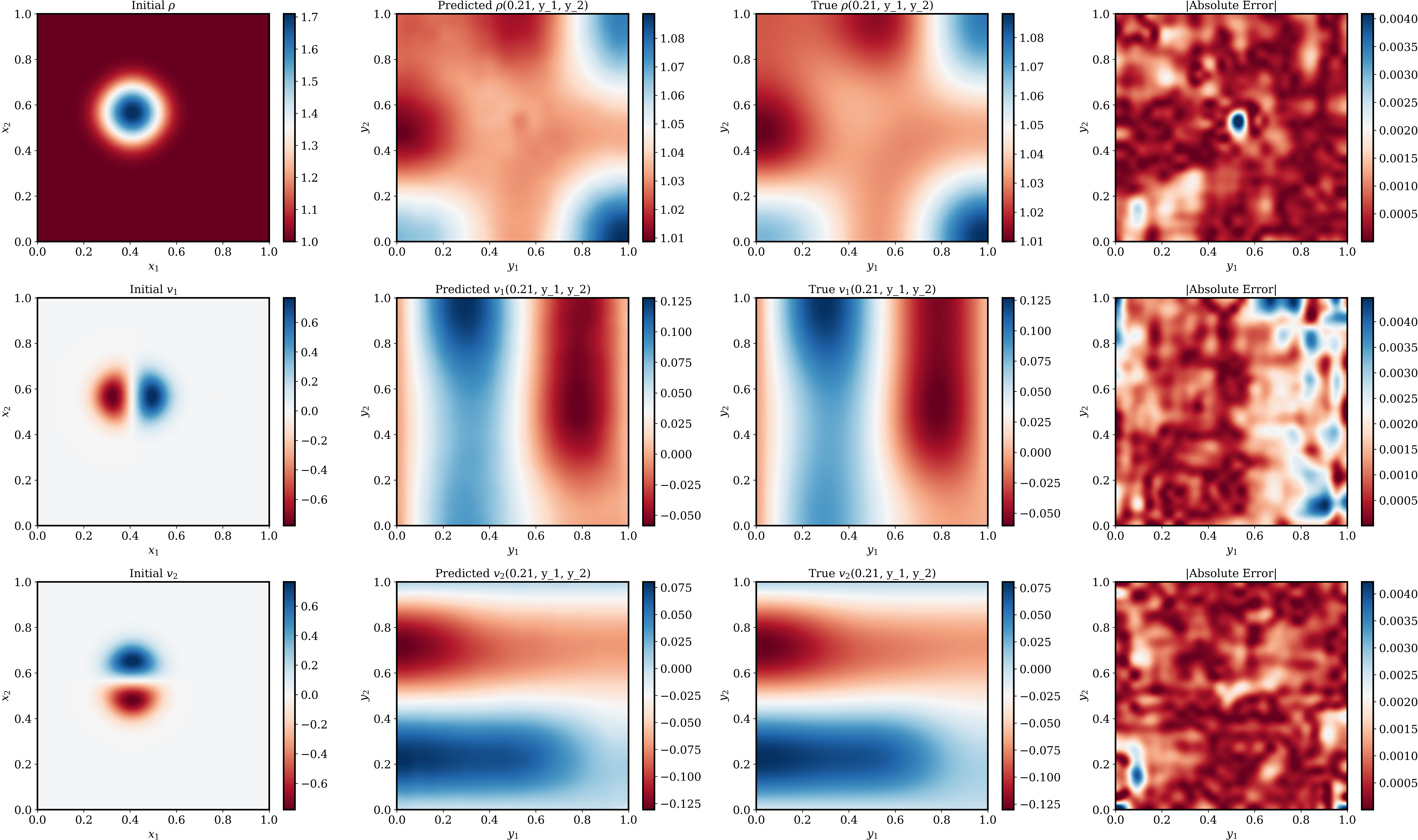}
    \caption{Qualitative comparison on the Shallow Water benchmark. Shown are the initial conditions, the predicted and ground-truth solutions, and the absolute error at 
    t=0.21s.} 
    
\end{figure*}

\begin{figure*}[t]
    \centering
    \includegraphics[width=\textwidth]{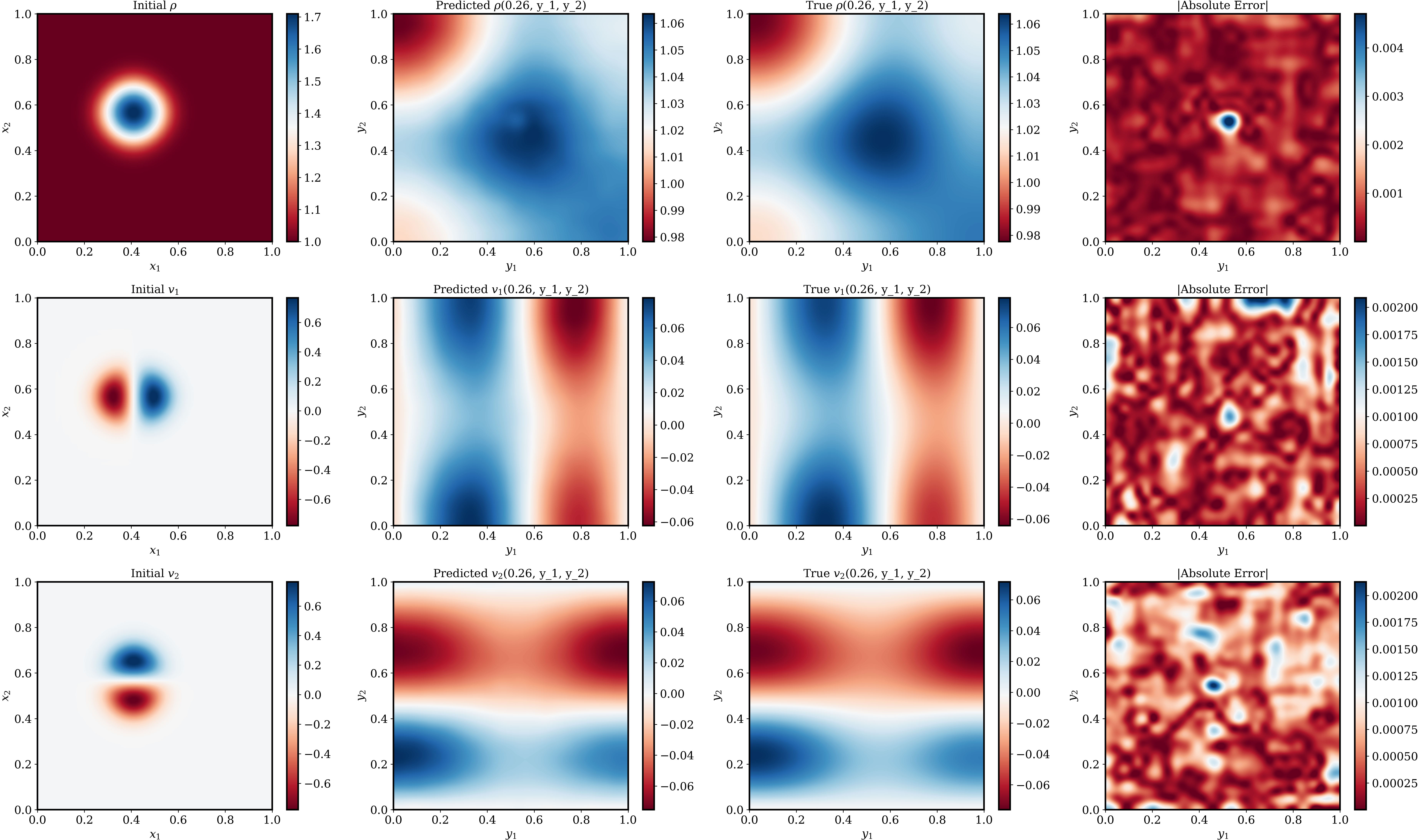}
    \caption{Qualitative comparison on the Shallow Water benchmark. Shown are the initial conditions, the predicted and ground-truth solutions, and the absolute error at 
    t=0.26s.} 
    
\end{figure*}

\end{document}